%% file: main.tex
\documentclass{article} 
\usepackage{iclr2021_conference,times}

\input{yan_pack}

\usepackage{hyperref}
\hypersetup{
     colorlinks=true,
     linkcolor=blue,
     citecolor = blue,      
     urlcolor=cyan,
}
\usepackage[toc,page,header]{appendix}
\usepackage{minitoc}

\usepackage{cleveref}
\title{Boosting One-Point Derivative-Free Online Optimization via Residual Feedback
}

\author{%
	Yan Zhang\thanks{Equal contribution} \\
	Dept. of Mechanical Eng. and Material Science\\
	Duke University\\
	Durham, NC 27708 \\
	\texttt{yan.zhang2@duke.edu} \\
	\And
	Yi Zhou\textsuperscript{*} \\
	Dept. of Electrical \& Computer Eng.\\
	The University of Utah \\
	Salt Lake City, UT 84112\\
	\texttt{yi.zhou@utah.edu} \\
	\And
	Kaiyi Ji \\
	Dept. of Electrical \& Computer Eng. \\
	The Ohio State University\\
	Columbus, OH 43210 \\
	\texttt{ji.367@osu.edu}\\
	\And 
	Michael M. Zavlanos \\
	Dept. of Mechanical Eng. and Material Science\\
	Duke University\\
	Durham, NC 27708 \\
	\texttt{michael.zavlanos@duke.edu} \\
}

\iclrfinalcopy 
\begin{document}

\doparttoc 
\faketableofcontents 
\part{} 

\maketitle

\begin{abstract}
Zeroth-order optimization (ZO) typically relies on two-point feedback to estimate the unknown gradient of the objective function. Nevertheless, two-point feedback can not be used for online optimization of time-varying objective functions, where only a single query of the function value is possible at each time step.	
In this work, we propose a new one-point feedback method for online optimization that estimates the objective function gradient using the residual between two feedback points at consecutive time instants. Moreover, we develop regret bounds for ZO with residual feedback for both convex and nonconvex online optimization problems. Specifically, for both deterministic and stochastic problems and for both Lipschitz and smooth objective functions, we show that using residual feedback can produce gradient estimates with much smaller variance compared to conventional one-point feedback methods. As a result, our regret bounds are much tighter compared to existing regret bounds for ZO with conventional one-point feedback, which suggests that ZO with residual feedback can better track the optimizer of online optimization problems. Additionally, our regret bounds rely on weaker assumptions than those used in conventional one-point feedback methods. Numerical experiments show that ZO with residual feedback significantly outperforms existing one-point feedback methods also in practice.
\end{abstract}

\input{introduction}

\input{preliminaries}

\input{noiseless}

\input{noisy_online}

\input{experiments}

\input{conclusion}

\bibliography{biblio}
\bibliographystyle{iclr2021_conference}

\newpage
\appendix

\addcontentsline{toc}{section}{Appendix} 
\part{Appendix} 
\parttoc 

\input{Online_Proof}

\input{Online_Stochastic_Proof}

\input{Additional_Proof}

\end{document}

%% file: yan_pack.tex
\usepackage[utf8]{inputenc} 
\usepackage[T1]{fontenc}    
\usepackage{hyperref}       
\usepackage{url}            
\usepackage{booktabs}       
\usepackage{nicefrac}       
\usepackage{microtype}      

\usepackage{amsmath,amsfonts,stmaryrd,amssymb} 
\usepackage{enumerate} 
\usepackage{enumitem} 

\usepackage{multirow}

\usepackage{amsthm}
\newtheorem{thm}{Theorem}[section]

\newtheorem{lem}[thm]{Lemma}
\newtheorem{asmp}[thm]{Assumption}

\newtheorem{defn}[thm]{Definition}
\newtheorem{rem}[thm]{Remark}

\usepackage{hyperref,cleveref}

\usepackage[linesnumbered,ruled,vlined]{algorithm2e}

\usepackage{xcolor}

\usepackage{hyperref}

\usepackage{graphicx}
\usepackage{wrapfig}
\usepackage{subfigure}

\usepackage{footmisc}


%% file: introduction.tex
\section{Introduction}\label{sec:Intro}
 Zeroth-order optimization (ZO) algorithms have been widely used to solve online optimization problems where first or second order information (i.e., gradient or Hessian information) is unavailable at each time instant. Such problems arise, e.g., in online learning and involve adversarial training \cite{chen2017zoo} and reinforcement learning \cite{fazel2018global,malik2018derivative} among others. The goal is to minimize a sequence of time-varying objective functions $\{f_t(x)\}_{t = 1 : T}$, 
 where the value $f_t(x_t)$ is revealed to the agent after an action $x_t$ is selected and is used to adapt the agent's future strategy. Since the objective functions are not known {\em a priori}, the quality of an online decision can be measured using notions of regret, that generally compare the total cost incurred by an online decision to the cost of the fixed or varying optimal decision that a clairvoyant agent could select.
 
Perhaps the most popular zeroth-order gradient estimator is the two-point estimator that has been extensively studied in \cite{agarwal2010optimal, ghadimi2013stochastic,duchi2015optimal,bach2016highly,nesterov2017random,gao2018online,roy2019multi}. This estimator queries the function value $f_t(x)$ twice at each time step, and uses the difference in the two function values to estimate the desired gradient, i.e.,
 \begin{align} \label{eqn:TwoPoint}
\text{(Two-point} &~ \text{feedback):} \;\; \widetilde{g}_t^{(2)}(x) = \frac{u}{\delta} \Big(f_t(x + \delta u) - f_t(x)\Big),
\end{align} 
where $\delta>0$ is a parameter and $u\sim \mathcal{N}(0,I)$.
Although this two-point estimator produces gradient estimates with low variance that improve the convergence speed of ZO, it can not be used for non-stationary online optimization problems that arise frequently in online learning.
The reason is that in these non-stationary online optimization problems, the objective function being queried is time-varying, and hence only a single function value can be sampled at a given time instant.
In this case, one-point estimators can be used instead that query the objective function $f_t(x)$ only once at each time instant, i.e., 
 \begin{align} \label{eqn:OnePoint}
	 \text{(One-point} &~ \text{feedback):} \;\; \widetilde{g}_t^{(1)}(x) = \frac{u}{\delta} f_t(x + \delta u).
 \end{align}
One-point feedback was first proposed and analyzed in \cite{flaxman2005online} for convex online optimization problems.
\cite{saha2011improved,dekel2015bandit} showed that the regret of convex online optimization methods using one-point gradient estimation can be improved if the objective functions are assumed to be smooth and self-concordant regularization is used. More recently, \cite{gasnikov2017stochastic}
developed regret bounds for ZO with one-point feedback also for stochastic convex problems. 
On the other hand, \cite{hazan2016graduated} characterized the convergence of one-point zeroth-order methods for static stochastic non-convex optimization problems.
However, as shown in these studies, one-point feedback produces gradient estimates with large variance which results in increased regret. In addition, the regret analysis for ZO with one-point feedback usually requires the strong assumption that the function value is uniformly upper bounded over time, so this method can not be used for practical non-stationary optimization problems.

{\bf Contributions:} In this paper, we propose a novel one-point gradient estimator for zeroth-order online optimization and develop new regret bounds to study its performance. Our proposed estimator uses the residual between two consecutive feedback points to estimate the gradient and, therefore, we refer to it as residual feedback. 
We show that, for both deterministic and stochastic problems, using residual feedback produces gradient estimates with lower variance compared to those produced using the conventional one-point feedback proposed in \cite{flaxman2005online,gasnikov2017stochastic}. As a result, we obtain tighter regret bounds both for convex and non-convex problems, especially when the value of the objective function is large. Moreover, our regret analysis relies on weaker assumptions compared to those for ZO with conventional one-point feedback. 
Finally, we present numerical experiments that demonstrate that ZO with residual feedback significantly outperforms the conventional one-point method in its ability to track the time-varying optimizers of online learning problems. To the best of our knowledge, this is the first time a one-point zeroth-order method is theoretically studied for non-convex online optimization problems. 
It is also the first time that a one-point gradient estimator demonstrates comparable empirical performance to that of the two-point method.
We note that two-point estimators can only be used to solve non-stationary online learning problems in simulation, 
where the system can be reset to the same fixed state during two different queries of the objective function values at a given time instant.
{\bf Related work:} 
Online optimization problems are only one instance of optimization problems that ZO methods have been used to solve.
 For example, 
\cite{balasubramanian2018zeroth} apply ZO to solve a set-constrained optimization problem where the projection onto the constraint set is non-trivial. \cite{gorbunov2018accelerated,ji2019improved} apply a variance-reduced technique and acceleration schemes to achieve better convergence speed in ZO. \cite{wang2018stochastic} improve the dependence of the iteration complexity on the dimension of the problem under an additional sparsity assumption on the gradient of the objective function. Finally, \cite{hajinezhad2018gradient,tang2019distributed} apply zeroth-order oracles to distributed optimization problems when only bandit feedbacks are available at each local agents. Our proposed residual feedback oracle  can be used to solve such optimization problems as well. Also related is work by \cite{zhang2015online} that considers non-convex online bandit optimization problems with a single query at each time step. However, this method employs the exploration and exploitation bandit learning framework and the proposed analysis is restricted to a special class of non-convex objective functions. Finally, \cite{agarwal2011stochastic,hazan2016optimal,bubeck2017kernel} study online bandit algorithms using ellipsoid methods. In particular, these methods induce heavy computation per step and achieve regret bounds that have bad dependence on the problem dimension. As a comparison, our one-point method is computation light and achieves regret bounds that have better dependence on the problem dimension.

%% file: preliminaries.tex
\section{Preliminaries and Residual Feedback}
\label{sec:prelim}


In this section we provide basic definitions and results on ZO that will be needed in the subsequent analysis. We also define the residual feedback gradient estimator that we propose to solve online optimization problems with unknown gradient information. First, we define the class of Lipschitz and smooth objective functions we are concerned with.
\begin{defn}[Lipschitz functions]
	The class of Lipschtiz-continuous functions $C^{0,0}$ satisfies: for any $f \in C^{0,0}$, $|f(x) - f(y)| \le L_0 \|x-y\|, ~\forall x,y\in \mathbb{R}^d$,
	where $L_0>0$ is the Lipschitz parameter. The class of smooth functions $C^{1,1}$ satisfies: for any $f \in C^{1,1}$, $\|\nabla f(x) - \nabla f(y)\| \le L_1\|x-y\|, ~\forall x,y\in \mathbb{R}^d,$
	where $L_1>0$ is the smoothness parameter.
\end{defn}


The key idea in ZO is to estimate the unknown first-order gradient of the objective function $f$ using zeroth-order oracles that perturb the objective function around the current point along all directions uniformly. The ability of these oracles to correctly estimate the gradient is typically analyzed using the Gaussian-smoothed version of the function $f$ defined as $f_\delta(x) := \mathbb{E}_{u\sim \mathcal{N}(0,1)}[f(x+\delta u)]$, where the coordinates of the vector $u$ are i.i.d standard Gaussian random variables; see \cite{nesterov2017random}.
The following result bounds the approximation error of the function $f_\delta(x)$ and can be found in \cite{nesterov2017random}.
\begin{lem}
	\label{lem:GaussianApprox}
	Consider a function $f$ and its smoothed version $f_\delta$. It holds that 
	\begin{align*}
	|f_\delta(x)-f(x)|\le
	\begin{cases}
	\delta L_0\sqrt{d}, ~\text{if}~f\in C^{0,0}, \\
	\delta^2 L_1 d, ~\text{if}~f\in C^{1,1},
	\end{cases} \text{and }
	\|\nabla f_\delta(x)- \nabla f(x)\|\le \delta L_1 (d+3)^{3/2}, ~\text{if}~ f\in C^{1,1}.
	\end{align*}
\end{lem}
The smoothed function $f_\delta(x)$ also satisfies the following amenable property; see \cite{nesterov2017random}.

\begin{lem}
	\label{lem:SmoothedFunctionLipschitiz}
	If $f\in C^{0,0}$ is $L_0$-Lipschitz, then $f_\delta\in C^{1,1}$ with Lipschitz constant $L_1 = \sqrt{d}\delta^{-1}L_0$.
\end{lem}

In this paper we consider the following online bandit optimization problem
\begin{equation}
\min_{x \in \mathcal{X}} \sum_{t=0}^{T-1} f_t(x), \tag{P}
\end{equation}
where $\mathcal{X}\subset \mathbb{R}^d$ is a convex set and $\{f_t\}_t$ is a random sequence of objective functions.
We assume that at time $t$, a new objective function $f_t$ is randomly generated independent of an agent's decisions, the objective functions $\{f_t\}_t$ are unknown {\em a priori} and their derivatives are unavailable
but can be estimated using a zeroth-order oracle that queries the objective function value at different perturbed points $x$, as discussed above. The goal is to determine an online decision $x$ with cost that is as close as possible to the cost of a fixed or varying optimal decision that a clairvoyant agent could select, which is measured using notions of regret. 


Such online optimization problems often arise in non-stationary learning, where the system is time-varying or a single query of the function $f_t$ changes the system state (i.e., $f_t$ changes to $f_{t+1}$). In these problems, two-point feedback can not be used to estimate the unknown gradient as it requires to evaluate $f_t$ at two different points at the same time $t$. Instead, a more practical approach is to use the one-point feedback scheme~\eqref{eqn:OnePoint} in \cite{gasnikov2017stochastic}.
However, the gradient estimates produced by the one-point feedback method in \eqref{eqn:OnePoint} have large variance that leads to large regret and, therefore, poor ability to track the optimizer of the online problem. To address this limitation, in this paper we propose a novel one-point gradient estimator, which we call a one-point residual feedback estimator, that has reduced variance and is defined as
\begin{align}
\text{(Residual feedback): } \;\;\; \widetilde{g}_t(x_t) := \frac{u_t}{\delta}\big(f_t(x_t + \delta u_t) - f_{t-1}(x_{t-1} + \delta u_{t-1})\big), \label{eqn:GradientEstimate_Noiseless}
\end{align}
where $u_{t-1},u_t \sim \mathcal{N}(0,I)$ are independent random vectors. To elaborate, the proposed residual feedback estimator in \eqref{eqn:GradientEstimate_Noiseless} queries $f_t$ at a single perturbed point $x_t + \delta u_t$, and then subtracts the value $f_{t-1}(x_{t-1} + \delta u_{t-1})$ obtained from the previous iteration. 
Next, we discuss some basic properties of this new estimator. We first show that this estimator provides an unbiased gradient estimate of the smoothed function $f_{\delta,t}$.
\begin{lem}
	\label{lem:UnbiasedEstimate_Noiseless}
	The residual feedback estimator satisfies $\mathbb{E}\big[\widetilde{g}_t(x_t)\big] = \nabla f_{\delta,t}(x_t)$ for all $x_t \in  \mathcal{X}$ and $t$.
\end{lem}
\begin{proof}
	The proof follows from the fact that $u_{t}$ has zero mean and is independent from $u_{t-1}$ and $x_{t-1}$. 
\end{proof}
{\color{red}
	\begin{rem}
		We note that existing two-point estimators can not be easily modified to be used for non-stationary optimization. The difficulty is in ensuring that the returned gradient estimates are unbiased as in the case of residual feedback in Lemma 2.4. To see this, consider the simple modification of the online two-point gradient estimator (7) proposed in \cite{bach2016highly}
		\begin{align} \label{eqn:modified_TP}
		\tilde{g}_t(x_t) = \frac{u_t}{2\delta}\big(f_t(x_t + \delta u_t) - f_{t-1}(x_t - \delta u_t)\big). \nonumber
		\end{align}
		Then, it is easy to see that this modified two-point gradient estimator is biased since  $\mathbb{E}\big[\tilde{g}_t(x_t)\big] \neq \nabla f_{\delta, t}(x_t)$. 
		Specifically, let $\tilde{g}_t(x_t) = \frac{u_t}{2\delta}\big(f_t(x_t + \delta u_t) - f_{t-1}(x_t - \delta u_t)\big) = \frac{u_t}{2\delta}\big(f_t(x_t + \delta u_t) - f_{t}(x_t - \delta u_t) + \epsilon_t\big)$, where $\epsilon_t = f_{t}(x_t - \delta u_t) - f_{t-1}(x_t - \delta u_t)$. Although $\mathbb{E}\big[  \frac{u_t}{2\delta}\big(f_t(x_t + \delta u_t) - f_{t}(x_t - \delta u_t)\big) \big] = \nabla f_{\delta, t}(x_t)$, we have that $\mathbb{E}\big[ \frac{u_t}{2\delta} \epsilon_t \big] \neq 0$ since $\epsilon_t$ is correlated with $u_t$. Therefore, for this modified estimator we have that  $\mathbb{E}\big[\tilde{g}_t(x_t)\big] = \mathbb{E}\big[  \frac{u_t}{2\delta}\big(f_t(x_t + \delta u_t) - f_{t}(x_t - \delta u_t)\big) \big] + \mathbb{E}\big[ \frac{u_t}{2\delta} \epsilon_t \big] \neq \nabla f_{\delta, t}(x_t)$. Note that the original two-point estimator proposed in \cite{bach2016highly} is unbiased, because the function $f_t$ is queried at two points, $x_t + \delta u_t$ and $x_t - \delta u_t$, and the noise $\epsilon_t$ in this case is simply the evaluation noise that is zero mean for any $u_t$. 
	\end{rem}
}
In this paper, we consider the following ZO projected gradient update with residual feedback to solve the online problem (P):
\begin{equation}
\label{eqn:SGD}
\text{(ZO with residual feedback):}\quad x_{t+1} = \Pi_{\mathcal{X}} \big(x_t - \eta \tilde{g}_t (x_t) \big),
\end{equation}
where $\eta$ is the learning rate and $\Pi_{\mathcal{X}}$ is the projection operator onto the set $\mathcal{X}$. 
The update~\eqref{eqn:SGD} can be implemented assuming that the objective function can be queried at points outside the feasible set $\mathcal{X}$, similar to the methods considered in \cite{duchi2015optimal,bach2016highly,gasnikov2017stochastic}. Note that it is possible to modify the update~\eqref{eqn:SGD} so that the iterates are guaranteed to be within the feasible set $\mathcal{X}$. This modification and related analysis can be found in Section~\ref{sec:UniformSample} in the supplementary material. The requirement that the objective function is evaluated at feasible points in derivative-free optimization algorithms has also been considered in \cite{bubeck2017kernel,bilenne2020fast}. Specifically, \cite{bubeck2017kernel} develop the so called ellipsoid method, which requires computation of an ellipsoid containing the optimizer at each time step. On the other hand, almost concurrently with this work, \cite{bilenne2020fast} proposed a similar oracle as in \eqref{eqn:GradientEstimate_Noiseless} for a static convex optimization problem with specific objective and constraint functions.
The following result bounds the second moment of the gradient estimate generated by using residual feedback.

\begin{lem}[Second moment]\label{lem:BoundSecondMoment_Det}
	Assume that $f_t \in C^{0,0}$ with Lipschitz constant $L_0$ for all time $t$. Then, under the ZO update rule in~\eqref{eqn:SGD}, the second moment of the residual feedback satisfies:
	\begin{align} \label{eqn:SecondMomentBound}
	\mathbb{E}[\|\widetilde{g}_t (x_t)\|^2] &\leq \; \frac{4 d L_0^2 \eta^2}{\delta^2} \mathbb{E}[ \|\widetilde{g}_{t-1}(x_{t-1})\|^2] +D_t,
	\end{align}
	where $D_t := 16L_0^2 (d+4)^2 + \frac{2d}{\delta^2} \mathbb{E} \big[ \big( f_t(x_{t-1} + \delta u_{t-1}) - f_{t-1}(x_{t-1} + \delta u_{t-1})\big)^2\big]$.
\end{lem}

The proof of above lemma can be found in Appendix~\ref{sec:proof_BoundSecondMoment}. The above lemma shows that the second moment of the gradient estimates obtained using residual feedback forms a contraction with perturbation term $D_t$, provided that we choose $\eta$ and $\delta$ such that the contracting rate satisfies $\alpha = 4 d L_0^2 \eta^2 {\delta^{-2}} < 1$. As we show later in the analysis, this contraction property leads to gradient estimates with low variance that allow to reduce the regret of the online ZO algorithm~\eqref{eqn:SGD}.

%% file: noiseless.tex
\section{ZO with Residual Feedback for Convex Online Optimization}\label{sec:det}
\label{sec:Deterministic_ConvexOnline}

In this section, we consider the online bandit problem (P) where the sequence of functions $\{f_t\}_{t=0:T-1}$ are all convex. 
In particular, we are interested in analyzing the static regret of algorithm \eqref{eqn:SGD} defined as
\begin{align}
	R_T := \mathbb{E} \Big[ \sum_{t=0}^{T-1} f_{t}(x_t) - \min_{x \in \mathcal{X}} \sum_{t=0}^{T-1} f_{t}(x) \Big].
\end{align}

First, we make the following assumption on the non-stationarity of the online learning problem.
\begin{asmp}[Bounded variation] \label{asmp:BoundVariation}
	There exists $V_f > 0$ such that for all $t$,
	\begin{align}
	\mathbb{E} \big[ | f_t(x_{t-1} + \delta u_{t-1}) - f_{t-1}(x_{t-1} + \delta u_{t-1})|^2 \big] \leq V_f^2,
	\end{align}
	where the expectation is taken over $x_{t-1}$, the random vector $u_{t-1}$ and the random functions $f_{t-1}$,$f_t$.
\end{asmp}
Assumption~\ref{asmp:BoundVariation} states that the squared variation of the objective function between two consecutive time instants is uniformly bounded over time. We note that this assumption is weaker than the assumption that the objective function is uniformly bounded, i.e., $|f_t(x)| \leq B, \forall t,x$, which is used in the analysis of ZO with conventional one-point feedback in  \cite{flaxman2005online,gasnikov2017stochastic}. 
In particular, under Assumption~\ref{asmp:BoundVariation}, the perturbation term in Lemma \ref{lem:BoundSecondMoment_Det} can be bounded as $D_t \leq  16L_0^2 (d+4)^2 + 2dV_f^2{\delta^{-2}}$. 
Then, by telescoping the contraction inequality, we obtain the following bound for the second moment of the residual-feedback gradient estimate
\begin{align}
    \mathbb{E}[\|\tilde{g}_t (x_t)\|^2] \leq \max \Big\{ \mathbb{E}[\|\tilde{g}_0 (x_0)\|^2], \frac{1}{1-\alpha} \Big(16L_0^2 (d+4)^2 +  \frac{2d}{\delta^2} V_f^2 \Big) \Big\}. \label{eqn:secondmoment}
\end{align}
The detailed proof can be found in Appendix~\ref{sec:proof_SecondMoment}. In practice, $\delta$ needs to be sufficiently small so that the smoothed function $f_{\delta,t}$ is close to the original function $f_t$ according to Lemma~\ref{lem:GaussianApprox}.
In this case, the above bound on the second moment of the residual-feedback gradient estimates is dominated by $\mathcal{O}(d{\delta^{-2}} V_f^2)$, which is much smaller than the bound on the second moment of the conventional one-point gradient estimates $\mathcal{O}(d {\delta^{-2}}B^2)$, where $B$ is the uniform bound on $|f_t|$ over time. 
For example, consider the time-varying objective functions, $f_0(x) = 1/2x^2$ and $f_t(x) = f_{t-1}(x) + n_t$, where $n_t$ is Gaussian noise with zero mean at time $t$. Then, it can be verified that Assumption \ref{asmp:BoundVariation} holds with a finite $V_f$ whereas the second moment of $f_t(x)$ is unbounded over time.
As a result, the variance of the residual feedback gradient estimates can be significantly smaller than that of the conventional one-point feedback gradient estimates. 



The following result characterizes the regret of ZO with residual feedback when the objective function $f_t$ is convex and Lipschitz.
\begin{thm}[Regret for Convex Lipschitz $f_t$]\label{thm: convex_Lip}
	Let Assumption~\ref{asmp:BoundVariation} hold. Assume that $f_t\in C^{0,0}$ is convex with Lipschitz constant $L_0$ for all $t$ and $\|x_0 - x^\ast\| \leq R$. Run ZO with residual feedback for $T > R^2$ iterations with $\eta = R^{\frac{3}{2}} ({2\sqrt{2} L_0 \sqrt{d} T^{\frac{3}{4}}})^{-1}$ and $\delta = \sqrt{R} {T^{-\frac{1}{4}}}$. Then, we have that
	\begin{align}
	R_T \leq & \; \sqrt{2} L_0 \sqrt{dR} T^{\frac{3}{4}}  + \frac{\mathbb{E} \big[ \|\tilde{g}_0(x_0)\|^2 \big] R^{\frac{3}{2}}}{2 \sqrt{2d} L_0 T^{\frac{3}{4}}} + 8\sqrt{2} \frac{(d+4)^2}{\sqrt{d}} L_0 R^{\frac{3}{2}} T^{\frac{1}{4}} \nonumber \\
	&  + 2 L_0 \sqrt{dR} T^{\frac{3}{4}} + \sqrt{2dR} V_f^2 {L_0}^{-1} T^{\frac{3}{4}}.
	\end{align}
Asymptotically, we have $R_T = \mathcal{O}( (L_0 +{L_0}^{-1} V_f^2) \sqrt{dR}  T^{\frac{3}{4}})$.
\end{thm}

The proof can be found in Appendix~\ref{sec:proof_convexLip}. To the best of our knowledge, the best known regret for ZO with conventional one-point feedback is of the order $\mathcal{O}(\sqrt{dL_0RB} T^{\frac{3}{4}})$ \cite{gasnikov2017stochastic}. Therefore, our regret bound is tighter if the function variation satisfies $V_f^2 \le \mathcal{O}(B^{\frac{1}{2}} L_0^{\frac{3}{2}})$. 
Essentially, using the proposed residual feedback gradient estimator, the regret of ZO  no longer depends on the uniform bound of the function value, which can be very large in practice. Instead, our regret only relies on how fast the function varies over time.
Note that knowledge of the neighborhood $R$ in Theorem~\ref{thm: convex_Lip} allows to select the stepsize $\eta$ and the parameter $\delta$ so that a better regret rate can be achieved that depends on $R$. However, knowledge of $R$ is not required and ZO 
with residual feedback converges from any initial point $x_0$. When the parameter $R$ is unknown, we can choose $\eta = ({2\sqrt{2} L_0 \sqrt{d} T^{\frac{3}{4}}})^{-1}$ and $\delta = {T^{-\frac{1}{4}}}$ and obtain the regret bound $R_T \le \mathcal{O}(L_0 R^2 \sqrt{d} T^{\frac{3}{4}} +{L_0}^{-1} \sqrt{d} V_f^2 T^{\frac{3}{4}})$. The proof can be found in Appendix~\ref{sec:proof_convexLip}.

\begin{rem}
We note that the complexity bound in Theorem 3.2 generally depends on the values of the Lipschitz parameters $L_0$, $L_1$ and the constant $V_f^2$. Specifically, choose $\eta = R^{\frac{3}{2}}(2\sqrt{2} L_0 \sqrt{d} T^{\frac{3}{4}})^{-1}$ and $\delta = \sqrt{R}L_0^{-q}T^{-\frac{1}{4}}$ with $q>0$ as a tuning parameter, and we obtain that $R_T = \mathcal{O}((L_0 +{L_0}^{1-q} + L_0^{2q-1} V_f^2) \sqrt{dR}  T^{\frac{3}{4}})$ when $T \geq L_0^{2q}R^2$. If $L_0 < 1$, we can choose $q = 1$ to achieve the bound $R_T = \mathcal{O}((L_0 + L_0V_f^2) \sqrt{dR}  T^{\frac{3}{4}})$. On the other hand, if $L_0 \ge 1$, we can choose $q = 0$ to achieve the bound $R_T = \mathcal{O}( (L_0 +{L_0}^{-1} V_f^2) \sqrt{dR}  T^{\frac{3}{4}})$. We note that the dependence of the bounds in Theorems~\ref{thm: convex_smooth}, \ref{thm:Online_Nonconvex_Nonsmooth} and \ref{thm:Online_Nonconvex_smooth} on $L_0, L_1$ can also be optimized in a similar way by properly choosing $\delta$.
\end{rem}

Next, we present the regret of ZO with residual feedback when the objective function $f_t$ is convex and smooth. 

\begin{thm}[Regret for Convex Smooth $f_t$]\label{thm: convex_smooth}
	Let Assumption~\ref{asmp:BoundVariation} hold. Assume that $f_t(x)\in C^{0,0}\cap C^{1,1}$ is convex with Lipschitz constant $L_0$ and smoothness constant $L_1$ for all $t$, and assume that $\|x_0 - x^\ast\| \leq R$.
	Run ZO with residual feedback for $T > R^2$ iterations with $\eta = R^{\frac{4}{3}} ({2\sqrt{2} L_0 d^\frac{2}{3} T^{\frac{2}{3}}})^{-1}$ and $\delta = R^{\frac{1}{3}} {d^{-\frac{1}{6}} T^{-\frac{1}{6}}}$. Then, we have that
	\begin{align}
	R_T  \leq & \;  \sqrt{2} L_0 d^\frac{2}{3} R^{\frac{2}{3}} T^{\frac{2}{3}} + \frac{\mathbb{E} \big[ \|\tilde{g}_0(x_0)\|^2 \big] R^{\frac{4}{3}}}{2\sqrt{2} L_0 d^\frac{2}{3} T^{\frac{2}{3}}}
	+ 8\sqrt{2} L_0 \frac{(d+4)^2}{d^\frac{2}{3}} R^{\frac{4}{3}} T^\frac{1}{3}  \nonumber \\
	& +  2 L_1 d^\frac{2}{3}R^{\frac{2}{3}} T^\frac{2}{3} + \sqrt{2} {L_0}^{-1} d^\frac{2}{3} R^{\frac{2}{3}} V_f^2 T^\frac{2}{3}.
	\end{align}
Asymptotically, we have that $R_T= \mathcal{O}((L_0 + L_1 + {L_0}^{-1} V_f^2 ) (dRT)^\frac{2}{3})$.
\end{thm}

The proof can be found in Appendix~\ref{sec:proof_convexSmooth}. To the best of our knowledge, the best known regret for ZO with conventional one-point feedback for convex and smooth problems is of the order $\mathcal{O}( L_1^\frac{1}{3}(dRBT)^{\frac{2}{3}})$ \cite{gasnikov2017stochastic}. Therefore, our regret bound
is tighter if the function variation satisfies $V_f^2 \le \mathcal{O}(B^{\frac{2}{3}} L_1^{\frac{1}{3}}L_0)$. Our numerical experiments in Section~\ref{sec:exp} show that ZO with residual feedback always outperforms ZO with conventional one-point feedback in practice.


\section{ZO with Residual Feedback for Non-Convex Online Optimization}
\label{sec:Deterministic_NonconvexOnline}
In this section, we analyze the regret of ZO with residual feedback for the unconstrained online bandit problem (P) where the objective functions $\{f_t\}_{t=0,...,T-1}$ are non-convex. 
To the best of our knowledge, this is the first time that a one-point zeroth-order method is studied for non-convex online optimization.
Throughout this section, we make the following assumption on the objective functions.

\begin{asmp}\label{asmp:BoundAccumVariation}
	There exist $W_T, \widetilde{W}_T >0$ such that the following conditions hold for all $t$. 
	\begin{enumerate}[leftmargin=*,topsep=0pt,noitemsep]
		\item $\sum_{t=1}^{T}\mathbb{E} [ f_{\delta,t}(x_t) - f_{\delta, t-1}(x_t) ] \leq W_T$, where the expectation is taken with respect to $x_t$ and the random smoothed objective functions $f_{\delta, t-1}$, $f_{\delta, t}$.
		\item $\sum_{t=1}^{T}\mathbb{E} [ | f_t(x_{t-1} + \delta u_{t-1}) - f_{t-1}(x_{t-1} + \delta u_{t-1}) |^2 ] \leq \widetilde{W}_T \;$, where the expectation is taken with respect to $x_{t-1}$, the random vector $u_{t-1}$ and the random objective functions $f_{t-1}$, $f_{t}$.
	\end{enumerate}
\end{asmp}
The above two conditions in Assumption~\ref{asmp:BoundAccumVariation} measure the accumulated first-order and second-order function variations. A similar assumption is made in \cite{roy2019multi}. 

First, we consider the case where $\{f_t\}_t$ are nonconvex and Lipschitz continuous functions. 
Since the objective function $f_t$ is not necessarily differentiable, i.e., $\nabla f(t)$ is not well defined, we define the regret as the accumulated gradient of the smoothed function, i.e.,
$R_{g, \delta}^T := \sum_{t=0}^{T-1} \mathbb{E}[\| \nabla f_{\delta,t}(x_t) \|^2].$
In addition, similar to \cite{nesterov2017random}, we require that the smoothed function $f_{\delta, t}$ is close to the original function $f_t$ such that $|f_{\delta, t}(x) - f_t(x)| \leq \epsilon_f$ for all $t$. To satisfy this condition, we need to choose $\delta \leq (\sqrt{d}L_0)^{-1} \epsilon_f$ according to Lemma~\ref{lem:GaussianApprox}.
Then, we can show the following regret bound for ZO with residual feedback. 


\begin{thm}[Nonconvex Lipschitz $f_t$]\label{thm:Online_Nonconvex_Nonsmooth}
Let Assumptions~\ref{asmp:BoundAccumVariation} hold. Assume that $f_t \in C^{0,0}$ with Lipschitz constant $L_0$ and that $f_t$ is bounded below by $f_t^\ast$ for all $t$. Run ZO with residual feedback for $T>(d\epsilon_f)^{-1}$ iterations with $\eta = \epsilon_f^{\frac{3}{2}} ({  2\sqrt{2} L_0^2 d^{\frac{3}{2}} T^{\frac{1}{2}}})^{-1}$ and $\delta = \epsilon_f ({d^\frac{1}{2} L_0})^{-1}$. Then, we have that
	\begin{align} \label{eqn:Nonsmooth_2}
	R_{g, \delta}^T \leq & \; 2\sqrt{2}L_0^2 \big( \mathbb{E}[f_{\delta,0}(x_0)] - f_{\delta,T}^\ast + W_T \big) d^{\frac{3}{2}}{\epsilon_f^{-\frac{3}{2}}} T^{\frac{1}{2}} + \frac{\epsilon_f^{\frac{1}{2}} \mathbb{E} \big[ \|\tilde{g}_0(x_0)\|^2 \big]}{2\sqrt{2d T}} \nonumber \\
	& + 4\sqrt{2}L_0\epsilon_f^{\frac{1}{2}} \frac{(d+4)^2}{d^{\frac{1}{2}}} T^{\frac{1}{2}} + \frac{L_0^2}{\sqrt{2}} \frac{d^{\frac{3}{2}} \widetilde{W}_T}{\epsilon_f^{\frac{3}{2}} T^{\frac{1}{2}}}.
	\end{align}
Asymptotically, we have $R_{g, \delta}^T = \mathcal{O}( d^{\frac{3}{2}}L_0^2 \epsilon_f^{-\frac{3}{2}}(W_T + \widetilde{W}_T T^{-1}) T^{\frac{1}{2}} + {d^{\frac{3}{2}}} L_0  \epsilon_f^{\frac{1}{2}} T^{\frac{1}{2}} )$.
\end{thm}
The proof can be found in Appendix~\ref{sec:proof_nonconvexLip}. Theorem~\ref{thm:Online_Nonconvex_Nonsmooth} implies that the regret bound satisfies $R_{g,\delta}^T/T \to 0$ whenever $W_T=o(T^\frac{1}{2}\epsilon_f^{\frac{3}{2}})$ and $\widetilde{W}_T=o(T^{\frac{3}{2}}\epsilon_f^{\frac{3}{2}})$. In particular, if the bounded variation Assumption~\ref{asmp:BoundAccumVariation} holds, then we have $\widetilde{W}_T\le \mathcal{O}(TV_f^2)$, and it suffices to let $T^{-\frac{1}{2}}\epsilon_f^{-\frac{3}{2}} = o(1)$.  

Next, we assume that the objective functions $f_t$ in (P) are non-convex and smooth
and study the regret $R_{g}^T := \sum_{t=0}^{T-1} \mathbb{E}[\| \nabla f_{t}(x_t) \|^2]$.
Specifically, we provide the following regret bound for ZO with residual-feedback.

\begin{thm}[Nonconvex smooth $f_t$]\label{thm:Online_Nonconvex_smooth}
Let Assumptions~\ref{asmp:BoundAccumVariation} hold. Assume that $f_t \in C^{0,0}\cap C^{1,1}$ with Lipschitz constant $L_0$ and smoothness constant $L_1$ and that $f_t$ is bounded below by $f_t^\ast$ for all $t$. Run ZO with residual feedback for $T$ iterations with $\eta = ({2 \sqrt{2} L_{0} d^{\frac{4}{3}} T^{\frac{1}{2}}})^{-1}$ and $\delta = ({d^{\frac{5}{6}} T^\frac{1}{4}})^{-1}$. Then, 
	\begin{align} \label{eqn:Smooth_2}
	R_{g}^T  \leq &\; 4\sqrt{2} L_0 \big( \mathbb{E}[f_{\delta,0}(x_0)] - f_{\delta,T}^\ast + W_T \big) d^{\frac{4}{3}} T^{\frac{1}{2}} +  \frac{L_1 \mathbb{E} \big[ \|\tilde{g}_0(x_0)\|^2 \big] }{\sqrt{2} L_{0} d^{\frac{4}{3}} T^{\frac{1}{2}}} \nonumber \\
	& + 8\sqrt{2} L_1 L_0 \frac{(d+4)^2}{d^\frac{4}{3}} T^{\frac{1}{2}}
	+ \frac{\sqrt{2} L_1}{L_0}  d^{\frac{4}{3}} \widetilde{W}_T + 2 L_1^2 \frac{(d+3)^3}{d^{\frac{5}{3}}}  T^{\frac{1}{2}}.
	\end{align}
Asymptotically, we have that $R_g^T= \mathcal{O}(d^{\frac{4}{3}} L_0 W_T T^{\frac{1}{2}} + d^{\frac{4}{3}} L_1 {L_0}^{-1} \widetilde{W}_T)$.
\end{thm}

The proof can be found in Appendix~\ref{sec:proof_nonconvexSmooth}. Theorem~\ref{thm:Online_Nonconvex_smooth} implies that the regret bound satisfies $R_{g}^T/T \to 0$ whenever $W_T=o(T^\frac{1}{2})$ and $\widetilde{W}_T=o(T)$. We note that these requirements on $W_T,\widetilde{W}_T$ are weaker than those in the case of nonsmooth problems, as they do not rely on the small parameter $\epsilon_f$.


%% file: noisy_online.tex
\section{ZO with Residual Feedback for Stochastic Online Optimization}\label{sec:stochastic}
Our proposed residual feedback gradient estimator can be also extended to solve stochastic online bandit problems. Since the regret analysis is similar to that for deterministic online problems presented before, we only introduce the key technical lemmas and comment on the differences in the proof. Specifically, we consider the following stochastic online bandit problems 
\begin{equation}
\min_{x \in \mathcal{X}} \sum_{t=0}^{T-1} \mathbb{E}[F_t(x;\xi_t)], \quad \text{where}~ \mathbb{E}[ F_t(x;\xi_t)] = f_t(x), \forall t, \tag{R}
\end{equation}
where $\xi_t$ denotes a certain noise that is independent of $x$.
Different from the deterministic online problems discussed before, the agent here can only query noisy evaluations of the objective function. This covers scenarios where the agent does not have access to the underlying data distribution. 
To solve the above stochastic online problem, we propose the following stochastic residual feedback
\begin{equation}
\label{eqn:GradientEstimate_Noise}
\widetilde{g}_t(x_t) := \frac{u_t}{\delta} \big(F_t(x_t + \delta u_t ; \xi_t) - F_{t-1}(x_{t-1} + \delta u_{t-1} ; \xi_{t-1})\big),	
\end{equation}
where $\xi_{t-1}$ and $\xi_t$ are independent random samples that are sampled at consecutive iterations $t-1$ and $t$, respectively. 
Since the noisy function value $F(x ; \xi_t)$ is an unbiased estimate of the objective function $f_t(x)$, it is straightforward to show that \eqref{eqn:GradientEstimate_Noise} is an unbiased gradient estimate of the function $f_{\delta,t}(x)$. 
To analyze the regret of ZO with stochastic residual feedback, we first consider the  convex case and make the following assumption on the variation of the stochastic objective functions.
\begin{asmp}\label{asmp:BoundedVariance}
	(Bounded stochastic variation) There exists $V_{f, \xi} >0$ such that for all $t$,
	\vspace{-4pt}
	\begin{equation*}
	\label{eqn:BoundedVarianceFunc}
	\mathbb{E}\big[ \big(F_t(x_{t-1} + \delta u_{t-1}, \xi_t) - F_{t-1}(x_{t-1} + \delta u_{t-1}, \xi_{t-1}) \big)^2 \big] \leq V_{f, \xi}^2,
	\end{equation*}
	where the expectation is taken with respect to $x_{t-1}$, the random vector $u_{t-1}$ and the random objective functions $F_{t-1}(\cdot, \xi_{t-1})$, $F_{t}(\cdot, \xi_{t})$.
\end{asmp}
The above assumption generalizes Assumption~\ref{asmp:BoundVariation} to stochastic problems. The bound $V_{f, \xi}^2$ controls both the variation of function over time and the variation due to stochastic sampling. 

The following lemma characterizes the second moment of the stochastic residual feedback gradient estimates. Its proof can be found in Appendix~\ref{sec:proof_BSMStoch}.

\begin{lem}
	\label{lem:BoundSecondMoment_Stoch}
	Assume $F(x, \xi) \in C^{0,0}$ with Lipschitz constant $L_0$ for all $\xi$. Then, under the ZO update rule, we have that 
	\begin{equation*}
	\mathbb{E}[\|\widetilde{g}_t (x_t)\|^2] \leq  \frac{4 d L_0^2 \eta^2}{\delta^2} \mathbb{E}[ \|\widetilde{g}_t(x_{t-1})\|^2] + D_{t, \xi},
	\end{equation*}
	where $D_{t, \xi} := 16L_0^2 (d+4)^2 + \frac{2 d}{\delta^2} \mathbb{E}[ \big( F_t(x_{t-1} + \delta u_{t-1}, \xi_t) - F_{t-1}(x_{t-1} + \delta u_{t-1}, \xi_{t-1})\big)^2 ]$.
\end{lem}
Observe that the above second moment bound is very similar to that in Lemma \ref{lem:BoundSecondMoment_Det}, and the only difference is the perturbation term. Since the perturbation term $D_{t, \xi}$ can be further bounded by leveraging Assumption \ref{asmp:BoundedVariance}, the resulting second moment bound can become almost the same as that in \cref{eqn:secondmoment} for deterministic problems (simply replace $V_f$ in \cref{eqn:secondmoment} by $V_{f,\xi}$). Therefore, the regret analysis of ZO with stochastic residual feedback is the same as that of ZO with residual feedback for deterministic online problems. Consequently, ZO with stochastic residual feedback achieves almost the same regret bounds as those in Theorems \ref{thm: convex_Lip} and \ref{thm: convex_smooth}, and one simply needs to replace $V_f$ by $V_{f,\xi}$. 


In the case of non-convex stochastic online problems, we adopt the following assumption that generalizes Assumption \ref{asmp:BoundAccumVariation}.
\begin{asmp}\label{asmp:BoundAccumVariation_stochastic}
	There exists $W_{T}, \widetilde{W}_{T,\xi} >0$ such that the following two conditions hold for all $t$.
	\begin{enumerate}[leftmargin=*,topsep=0pt,noitemsep]
		\item $\sum_{t=1}^{T}\mathbb{E} [ f_{\delta,t}(x_t) - f_{\delta, t-1}(x_t) ] \leq W_T$, where the expectation is taken with respect to $x_t$ and the random smoothed objective functions $f_{\delta, t-1}$, $f_{\delta, t}$.
		
		\item $\sum_{t=1}^{T}\mathbb{E} [ | F_t(x_{t-1} + \delta u_{t-1};\xi_t) - F_{t-1}(x_{t-1} + \delta u_{t-1};\xi_{t-1}) |^2 ] \leq \widetilde{W}_{T,\xi}$, where the expectation is taken with respect to $x_{t-1}$, the random vector $u_{t-1}$ and the random objective functions $F_{t-1}(\cdot, \xi_{t-1})$, $F_{t}(\cdot, \xi_{t})$.
	\end{enumerate}
\end{asmp}
Then, following similar steps as those in the proofs of Theorems \ref{thm:Online_Nonconvex_Nonsmooth} and \ref{thm:Online_Nonconvex_smooth}, we can obtain similar regret bounds for ZO with stochastic residual feedback (simply replace $W_T, \widetilde{W}_{T}$ in Theorems \ref{thm:Online_Nonconvex_Nonsmooth} and \ref{thm:Online_Nonconvex_smooth} by $W_{T,\xi}, \widetilde{W}_{T,\xi}$, respectively).

%% file: experiments.tex

\section{Numerical Experiments}
\label{sec:exp}
In this section, we compare the performance of ZO with one-point, two-point and residual feedback in solving two non-stationary reinforcement learning problems, i.e., LQR control and resource allocation, in which either the reward or transition functions are varying over episodes. 
\subsection{Nonstatinoary LQR Control}
We consider an LQR problem with noisy system dynamics. 
The static version of this problem is considered in \cite{fazel2018global,malik2018derivative}. 
Specifically, consider a system whose state $x_k \in \mathbb{R}^{n_x}$ at step $k$ is subject to a transition function $x_{k+1} = A_t x_k + B_t u_k + w_k$,
where $u_k \in \mathbb{R}^{n_u}$ is the action at step $k$, and $A_t \in \mathbb{R}^{n_x \times n_x}$ and $B_t \in \mathbb{R}^{n_x \times n_u}$ are dynamical matrices in episode $t$. These matrices are unknown and changing over episodes. The vector $w_k$ is the noise on the state transition. Specifically, the entries of the dynamical matrices $A_0$ and $B_0$ at episode $0$ are randomly generated from a Gaussian distribution $\mathcal{N}(0, 0.1^2)$. Then, we generate the time-varying dynamical matrices as $A_{t+1} = A_t + 0.01 M_t$ and $B_{t+1} = B_t + 0.01N_t$, where $M_t$ and $N_t$ are random matrices whose entries are uniformly sampled from [0,1].
Moreover, consider a state feedback policy $u_k = K_t x_k$, where $K_t \in \mathbb{R}^{n_u \times n_x}$ is the policy parameter that is fixed during episode $t$. We assume that there exists an optimal policy $K_t^\ast$ so that the discounted accumulated cost function $V_t(K) := \mathbb{E}\big[ \sum_{k = 0}^{H-1} \gamma^{k} (x_k^T Q x_k + u_k^T R u_k)\big]$ at episode $t$ is minimized, where $\gamma \leq 1$ is the discount factor and $H$ is the horizon.
The goal is to track the time-varying optimal policy parameter $K_t^\ast$ so that $V_t(K_t) - V_t(K_t^\ast)$ is small in every episode.
\begin{figure}[t]
	\centering
	\subfigure[\label{fig:lqr}]{
		\includegraphics[width=0.47\columnwidth]{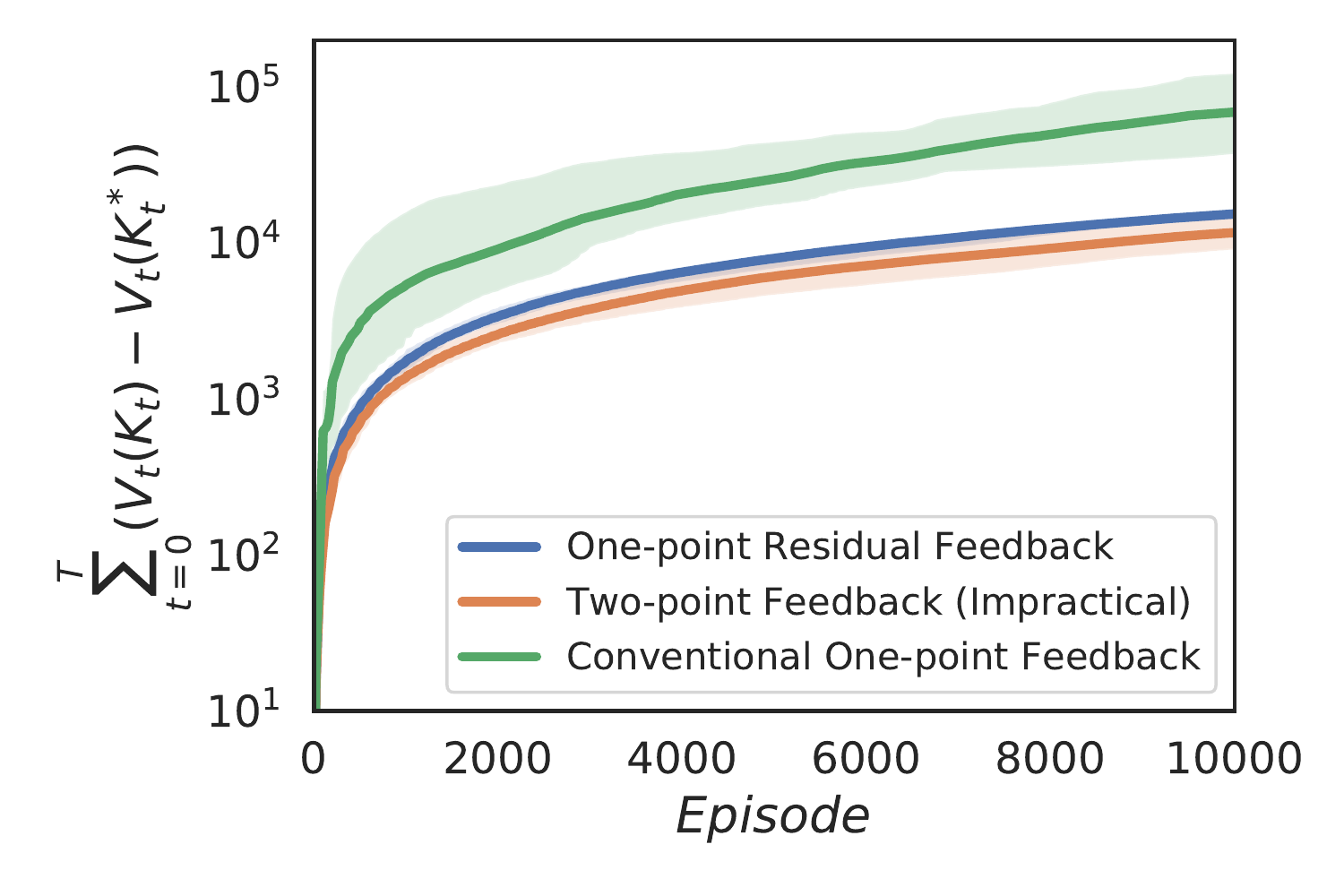}}
	\subfigure[\label{fig:var}]{
		\includegraphics[width=0.47\columnwidth]{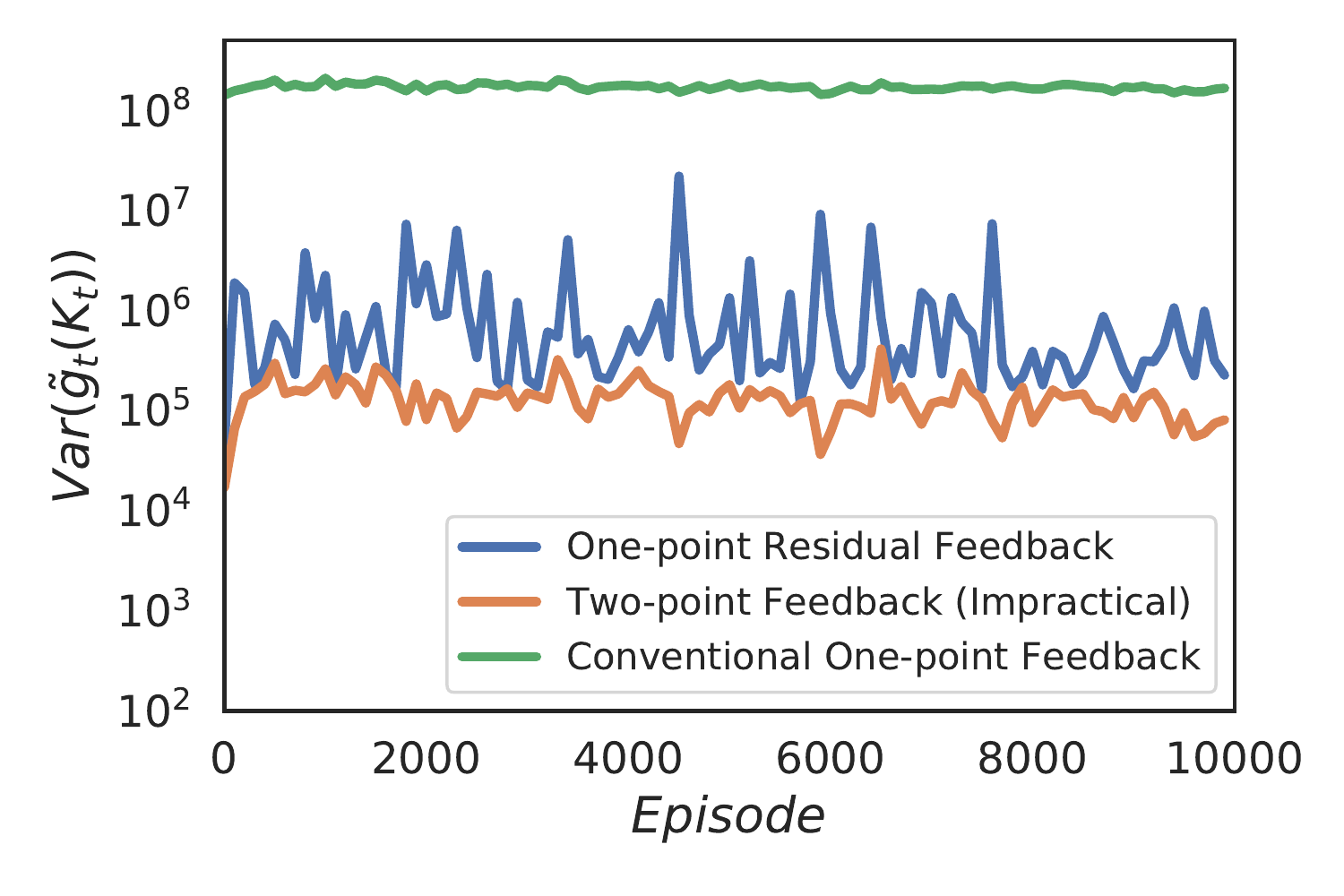}}
	\vspace{-2mm}
	\caption{\small Comparative results of ZO with the proposed one-point residual feedback~\eqref{eqn:GradientEstimate_Noiseless} (blue), the two-point oracle  in \cite{bach2016highly} (orange) and the conventional one-point oracle in \cite{gasnikov2017stochastic} (green) for online policy optimization in nonstationary LQR. Figure 1(a) presents the regrets $\sum_{t = 0}^{T}|V(K_t) - V(K^\ast)|$ achieved using the three diffident oracles and Figure 1(b) presents the variance of the gradient estimates returned by the three methods. The two point method (orange) is infeasible to use in practice and is presented here to serve as a simulation benchmark. }
\end{figure}

We apply the conventional one-point method in \cite{gasnikov2017stochastic} and the proposed residual-feedback method~\eqref{eqn:GradientEstimate_Noise} to solve the above non-stationary LQR problem. The performance of the two-point method in \cite{bach2016highly} is also presented to serve as a benchmark, although it is 
not possible to implement in practice for non-stationary problems. 
This is because the two-point method in \cite{bach2016highly} requires to evaluate value function $V_t$ for two different policy functions at two consecutive episodes. However, evaluating the value function $V_t$ for a given policy during episode $t$ requires to collect samples by executing this policy. Then, during the subsequent episode $t+1$, since the problem is non-stationary, the dynamic matrices change to $A_{t+1}, B_{t+1}$ and so does the value function $V_{t+1}$. Therefore, it is not possible to evaluate the same value function $V_t$ at two different episodes and, as a result, the two-point method in \cite{bach2016highly} is not applicable here. 
Each algorithm is run for $10$ trials, and the stepsizes are optimized for each algorithm separately. The accumulated regrets $\sum_{t = 0}^{T-1}|V(K_t) - V(K^\ast)|$ of the three algorithms are presented in \Cref{fig:lqr}. We observe that ZO with residual feedback achieves a much lower regret than the conventional one-point method and has a comparable performance to that of the two-point method. 
Moreover, we present in \Cref{fig:var} the estimated variance of the gradient estimates returned by these three oracles at the policy iterates over episodes. It can be seen that the variance of the gradient estimates returned by our proposed residual-feedback is close to that of the gradient estimates returned by the two-point feedback and is much smaller than that of the gradient estimates returned by the conventional one-point feedback. This observation validates our theoretical characterization of the second moment of the residual feedback gradient estimates. 
\subsection{Nonstationary Resource Allocation}
We consider a multi-stage resource allocation problem with time-varying sensitivity to the lack of resource supply. Specifically, $16$ agents are located on a $4 \times 4$ grid. During episode $t$, at step $k$, agent $i$ stores $m_i(k)$ amount of resources and has a demand for resources in the amount of $d_i(k)$.
Also, agent $i$ decides to send a fraction of resources $a_{ij}(k) \in [0, 1]$ to its neighbors $j \in  \mathcal{N}_i$ on the grid. The local amount of resources and demands of agent $i$ evolve as $m_i(k+1) = m_i(k) - \sum_{j \in \mathcal{N}_i} a_{ij}(k) m_i(k) + \sum_{j \in \mathcal{N}_i} a_{ji}(k) m_j(k) - d_i(k)$ and $d_i(k) = \psi_i \sin(\omega_i k + \phi_i) + w_{i,k}$,
where $w_{i,k}$ is the noise in the demand. At each step $k$, agent $i$ receives a local cost $r_{i,t}(k)$, such that $r_{i,t}(k) = 0$ when $m_i(k) \geq 0$ and $r_{i,t}(k) = \zeta_t m_i(k)^2$ when $m_i(k) < 0$, where $\zeta_t$ represents the varying sensitivity of the agents to the lack of supply during episode $t$. Let agent $i$ makes its decisions according to a parameterized policy function $\pi_{i, t}(o_i; \theta_{i,t}): \mathcal{O}_i \rightarrow [0,1]^{|\mathcal{N}_i|}$, where $\theta_{i,t}$ is the parameter of the policy function $\pi_{i,t}$ at episode $t$, $o_i \in \mathcal{O}_i$ denotes agent $i$'s local observation. Specifically, we let $o_i(k) = [m_i(k), d_i(k)]^T$. Our goal is to track the time-varying optimal policy so that the accumulated cost over the grid $J_t(\theta_t) = \sum_{i=1}^{16} \sum_{k = 0}^{H} \gamma^{k} r_{i,t}(k)$ during each episode is maintained at a low level, where $\theta_t = [\dots, \theta_{i,t}, \dots]$ is the policy parameter, $H$ is the problem horizon at each episode, and $\gamma$ is the discount factor.


In \Cref{fig:resource}, we present the cost $J_t(\theta_t)$ achieved during each episode after $10$ trials of ZO with residual-feedback, one-point, and two-point feedback which, as before, is impossible to use in practice for this non-stationary problem either.
It can be seen that ZO with our proposed residual-feedback achieves a cost $J_t(\theta_t)$ that is as low as the cost achieved by the two-point feedback in this non-stationary environment. In particular, ZO with both residual and two-point feedback performs much better than ZO with conventional one-point feedback. 
\Cref{fig:var_resource} also compares the estimated variance of the gradient estimates returned by these feedback schemes. It can be seen that the variance of the gradient estimates returned by the residual feedback oracle is comparable to that of the gradient estimates returned by the two-point oracle and is much smaller than that of the gradient estimates returned by the conventional one-point oracle.

\begin{figure}[t]
	\centering
	\subfigure[\label{fig:resource}]{
		\includegraphics[width=0.47\columnwidth]{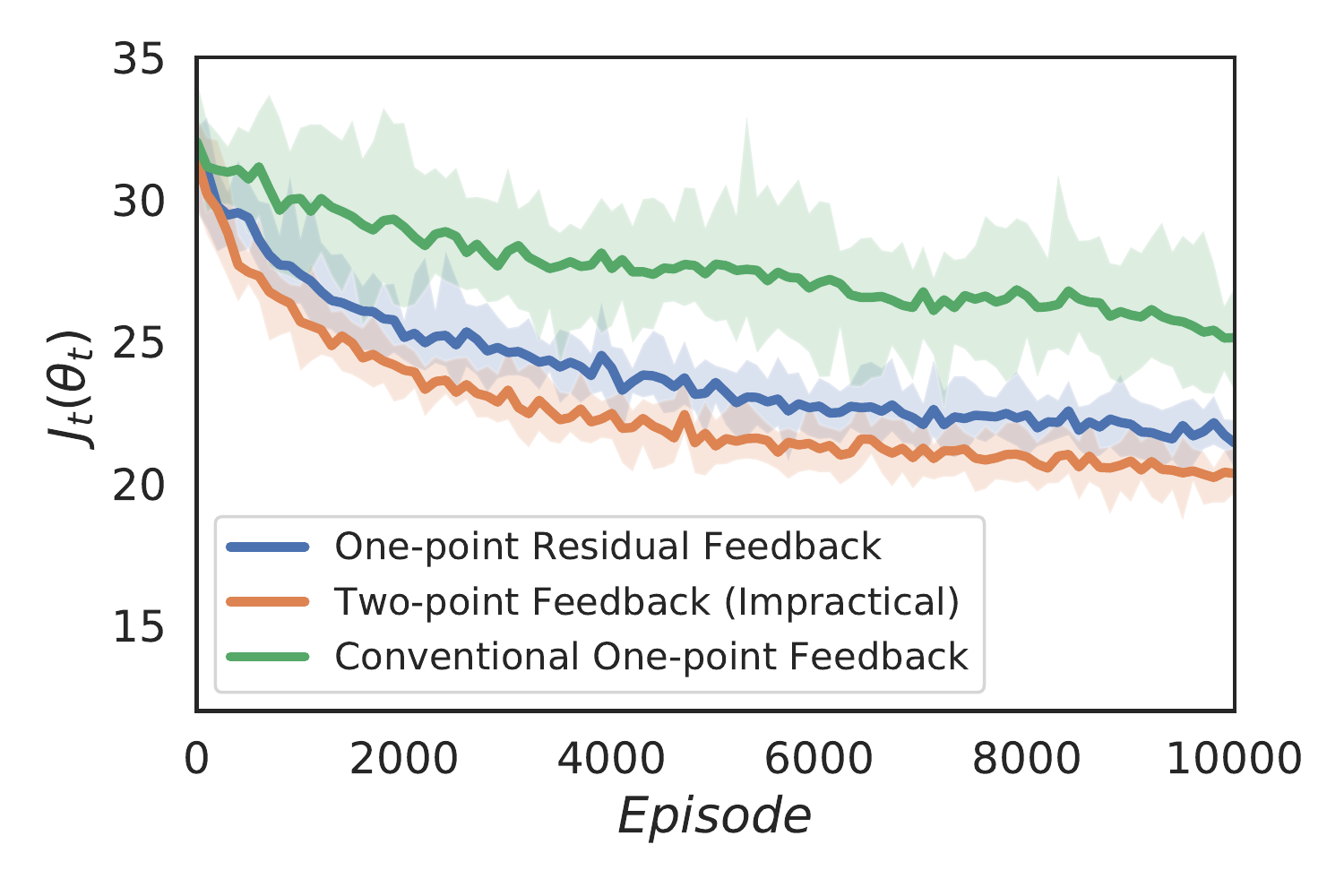}}
	\subfigure[\label{fig:var_resource}]{
		\includegraphics[width=0.47\columnwidth]{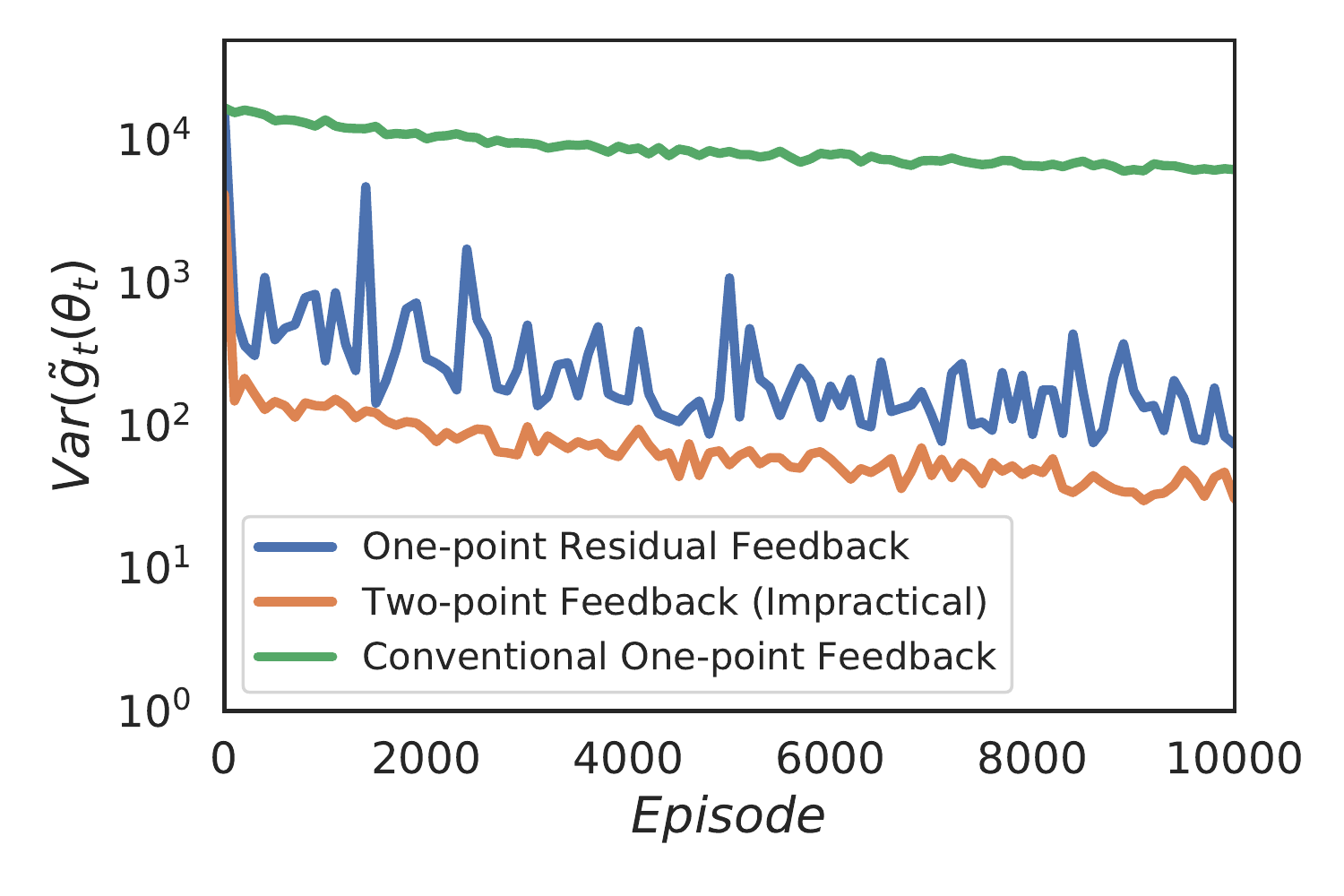}}
	\vspace{-2mm}
	\caption{\small 
		Comparative results of ZO with the proposed one-point residual feedback~\eqref{eqn:GradientEstimate_Noiseless} (blue), the two-point oracle in \cite{bach2016highly} (orange) and the conventional one-point oracle in \cite{gasnikov2017stochastic} (green) for the non-stationary resource allocation problem.  Figure 2(a) presents the varying cost $ J_t(\theta_t)$ achieved using the three diffident oracles and Figure 2(b) presents the variance of the gradient estimates at agent $1$ returned by the three methods. The two point method (orange) is infeasible to use in practice and is presented here to serve as a simulation benchmark.
	}
\end{figure}

%% file: conclusion.tex
\section{Conclusion}
In this paper, we proposed a novel one-point residual feedback oracle for zeroth-order online optimization, which estimates the gradient of the time-varying objective function using a single query of the function value at each time instant. For both deterministic and stochastic problems, we showed that ZO with the proposed residual feedback estimator achieves much lower regret than that of ZO with conventional one-point feedback for convex online optimization problems. 
In addition, we provided regret bounds for ZO with residual feedback for non-convex online optimization problems. To the best of our knowledge, this is the first time that a one-point zeroth-order method is theoretically studied for non-convex online problems. 
Numerical experiments on two non-stationary reinforcement learning problems were conducted and the proposed residual-feedback estimator was shown to significantly outperform the conventional one-point method.


%% file: Online_Proof.tex
\section{Implementation Details of the Numerical Experiments}
\label{sec:exp_detail}
All experiments are conducted using Matlab R2019a on Ubuntu 18.04 with the AMD Ryzen 2700X 8-core processor and 16GB 2133MHz memory.

For the non-stationary LQR experiments, we select $n_x = 6$, $n_u = 6$ and $\gamma = 0.5$. The dynamical matrices $A_0$ and $B_0$ at episode $0$ are randomly generated from a Gaussian distribution $\mathcal{N}(0, 0.1^2)$. Then, we generate the time-varying dynamical matrices according to $A_{t+1} = A_t + 0.01 M_t$ and $B_{t+1} = B_t + 0.01N_t$, where $M_t$ and $N_t$ are random matrices whose entries are uniformly sampled from [0,1]. To evaluate the cost function $V_t(K_t)$ given the policy parameter $K_t$ at episode $t$, we roll out a trajectory of length $H = 50$ using the policy parameter $K_t$ and sum up the collected rewards. 

For the non-stationary resource allocation experiments, the policy function $\pi_{i, t}(o_i; \theta_{i,t})$ is parameterized as: $a_{ij} = \exp(z_{ij}) / \sum_{j}\exp(z_{ij})$, where $z_{ij} = \sum_{p = 1}^{9} \psi_p(o_i) \theta_{ij}(p)$ and $\theta_i = [\dots, \theta_{ij}, \dots]^T$ and the episode index $t$ is omitted for notational simplicity. Specifically, the feature function $\psi_p(o_i)$ is selected as $\psi_p(o_i) = \|o_i - c_p\|^2$, where $c_p$ is the parameter of the $p$-th feature function.  Effectively, the agents need to make decisions on $64$ actions, and each action is decided by $9$ parameters. Therefore, the problem dimension is $d = 576$. The discount factor is selected as $\gamma = 0.75$ and the length of the horizon is $H = 30$. The time-varying sensitivity parameter $\zeta_{i,t}$ is generated as follows: let $\zeta_{i,0} = 1$ and $\zeta_{i,t+1} = \zeta_{i,t} + 0.1 P_t$, where $P_t$ is a random number uniformly sampled from $[-1, 1]$.

\section{Proof of Lemma \ref{lem:BoundSecondMoment_Det}}
\label{sec:proof_BoundSecondMoment}

By definition of the residual feedback, we have
\begin{equation}
\label{eqn:BSM_online_1}
\begin{split}
\mathbb{E}[\|\tilde{g}_t (x_t)\|^2] & = \mathbb{E}[\frac{1}{\delta^2} \big(f_t(x_t + \delta u_t) - f_{t-1}(x_{t-1} + \delta u_{t-1})\big)^2 \|u_t\|^2]  \\
& \leq \frac{2}{\delta^2} \mathbb{E}[ \big(f_t(x_t + \delta u_t) - f_t(x_{t-1} + \delta u_{t-1})\big)^2 \|u_t\|^2] \\
& \quad \quad \quad \quad \quad + \frac{2}{\delta^2} \mathbb{E}[ \big( f_t(x_{t-1} + \delta u_{t-1}) - f_{t-1}(x_{t-1} + \delta u_{t-1})\big)^2 \|u_t\|^2].
\end{split}
\end{equation}
Since $u_t$ is independent of $x_{t-1}$, $u_{t-1}$ and the generation of functions $f_{t-1}$ and $f_t$, we have that $\frac{2}{\delta^2} \mathbb{E}[ \big( f_t(x_{t-1} + \delta u_{t-1}) - f_{t-1}(x_{t-1} + \delta u_{t-1})\big)^2 \|u_t\|^2] \leq \frac{2d}{\delta^2} \mathbb{E}[ \big( f_t(x_{t-1} + \delta u_{t-1}) - f_{t-1}(x_{t-1} + \delta u_{t-1})\big)^2]$.
Moreover, adding and subtracting $f_t(x_{t-1} + \delta u_{t})$ in the term $\big(f_t(x_t + \delta u_t) - f_t(x_{t-1} + \delta u_{t-1})\big)^2$ of the above inequality, we obtain that
\begin{equation}
\begin{split}
\mathbb{E}[\|\tilde{g} (x_t)\|^2] \leq & \frac{4}{\delta^2} \mathbb{E}[ \big(f_t(x_t + \delta u_t) - f_t(x_{t-1} + \delta u_{t})\big)^2 \|u_t\|^2] \\
& + \frac{4}{\delta^2} \mathbb{E}[ \big( f_t(x_{t-1} + \delta u_{t}) - f_t(x_{t-1} + \delta u_{t-1})\big)^2 \|u_t\|^2] \\
& + \frac{2d}{\delta^2} \mathbb{E}[ \big( f_t(x_{t-1} + \delta u_{t-1}) - f_{t-1}(x_{t-1} + \delta u_{t-1})\big)^2].
\end{split}
\end{equation}
Since $f_t \in C^{0,0}$ is Lipschitz with constant $L_0$, we further obtain that
\begin{align}
\label{eqn:BSM_2}
\mathbb{E}[\|\tilde{g} (x_t)\|^2] \leq & \frac{4 L_0^2}{\delta^2} \mathbb{E}[ \|x_t - x_{t-1}\|^2 \|u_t\|^2 ] + 4 L_0^2 \mathbb{E}[\|u_t - u_{t-1}\|^2\|u_t\|^2] \nonumber \\
& + \frac{2d}{\delta^2} \mathbb{E}[ \big( f_t(x_{t-1} + \delta u_{t-1}) - f_{t-1}(x_{t-1} + \delta u_{t-1})\big)^2].
\end{align}
Note that $u_t$ is a Gaussian vector independent from $x_t - x_{t-1}$, we then obtain that $\mathbb{E}[ \|x_t - x_{t-1}\|^2 \|u_t\|^2 ] = d \mathbb{E}[\|x_t - x_{t-1}\|^2].$ Furthermore, using Lemma 1 in \cite{nesterov2017random}, we know that $\mathbb{E}[\|u_t - u_{t-1}\|^2\|u_t\|^2] \leq 2\mathbb{E}[(\|u_t\|^2 + \|u_{t-1}\|^2)\|u_t\|^2] = 2\mathbb{E}[(\|u_t\|^4] +2\mathbb{E}[ \|u_{t-1}\|^2\|u_t\|^2] \leq 4(d+4)^2$. Substituting these bounds into inequality~\eqref{eqn:BSM_2}, we obtain that
\begin{align}
\mathbb{E}[\|\tilde{g} (x_t)\|^2] \leq & \frac{4 d L_0^2}{\delta^2} \mathbb{E}[ \|x_t - x_{t-1}\|^2] + 16 L_0^2 (d+4)^2 \nonumber \\
& + \frac{2d}{\delta^2} \mathbb{E}[ \big( f_t(x_{t-1} + \delta u_{t-1}) - f_{t-1}(x_{t-1} + \delta u_{t-1})\big)^2]. \nonumber
\end{align}
Since $x_t = \Pi_\mathcal{X}\big[ x_{t-1} - \eta \tilde{g}(x_{t-1}) \big]$, we get that $\|x_t - x_{t-1}\| = \|\Pi_\mathcal{X}\big[ x_{t-1} - \eta \tilde{g}(x_{t-1}) \big] - \Pi_\mathcal{X}\big[ x_{t-1} \big]\| \leq \eta \|\tilde{g}(x_{t-1})\|$ due to the nonexpansiveness of the projection operator onto a convex set. Therefore, we have that
%
\begin{align}
\mathbb{E}[\|\tilde{g}_t (x_t)\|^2] \leq & \frac{4 d L_0^2 \eta^2}{\delta^2} \mathbb{E}[ \|\tilde{g}_{t-1}(x_{t-1})\|^2] + 16L_0^2 (d+4)^2 \nonumber \\
& + \frac{2d}{\delta^2} \mathbb{E}[ \big( f_t(x_{t-1} + \delta u_{t-1}) - f_{t-1}(x_{t-1} + \delta u_{t-1})\big)^2]. \nonumber
\end{align}
The proof is complete.

\section{Proof of Theorem \ref{thm: convex_Lip}}
\label{sec:proof_convexLip}

Note that $f_{\delta, t}(x)$ is convex for all $t$, we then conclude that
\begin{align}
f_{\delta,t}(x_t) - f_{\delta,t}(x) \leq \langle \nabla f_{\delta,t}(x_t), x_t - x \rangle, \text{ for all } x \in \mathcal{X},
\end{align}
Adding and subtracting $\tilde{g}_t(x_t)$ after $\nabla f_{\delta,t}(x_t)$ in above inequality, and taking expectation over $u_t$ on both sides, we obtain that
\begin{align}\label{eqn:BSM_Oline_1}
\mathbb{E} \big[ f_{\delta,t}(x_t) - f_{\delta,t}(x) \big] \leq \mathbb{E} \big[ \langle  \tilde{g}_t(x_t), x_t - x \rangle \big].
\end{align} 
Since $x_{t+1} = \Pi_\mathcal{X}\big[ x_{t} - \eta \tilde{g}(x_{t}) \big]$, for any $x\in \mathcal{X}$ we have that
\begin{align}
\|x_{t+1} - x\|^2 & =  \|\Pi_\mathcal{X}\big[ x_{t} - \eta \tilde{g}(x_{t}) \big] - \Pi_\mathcal{X}\big[x \big]\|^2 \nonumber \\
&\leq \| x_{t} - \eta \tilde{g}(x_{t}) - x \|^2 \nonumber \\
& = \|x_t - x\|^2 - 2 \eta \langle \tilde{g}_t(x_t), x_t - x \rangle + \eta^2 \|\tilde{g}_t(x_t)\|^2.
\end{align}
Rearranging the above inequality yields that
\begin{align}
\langle \tilde{g}_t(x_t), x_t - x \rangle = \frac{1}{2\eta} \big( \|x_t - x\|^2 -  \|x_{t+1} - x\|^2\big) + \frac{\eta}{2} \|\tilde{g}_t(x_t)\|^2.
\end{align}
Taking expectation on both sides of the above inequality over $u_t$ and substituting the resulting bound into \eqref{eqn:BSM_Oline_1}, we obtain that 
\begin{align}
\mathbb{E} \Big[ \sum_{t=0}^T f_{\delta,t}(x_t) - \sum_{t=0}^T f_{\delta,t}(x) \Big] \leq \frac{1}{2\eta} \|x_0 - x\|^2  + \frac{\eta}{2} \mathbb{E} \Big[ \sum_{t=0}^T \|\tilde{g}_t(x_t)\|^2 \Big].
\end{align}
Since $f_t(x) \in C^{0,0}$, we know that $|f_{\delta, t}(x) - f_t(x)| \leq \delta L_0 \sqrt{d}$. Therefore, we obtain from the above inequality that 
\begin{align} \label{eqn:BSM_Oline_2}
\mathbb{E} \big[ \sum_{t=0}^T f_{t}(x_t) - \sum_{t=0}^T f_{t}(x) \big] &= \mathbb{E} \big[ \sum_{t=0}^T f_{\delta,t}(x_t) - \sum_{t=0}^T f_{\delta,t}(x) \big] + \mathbb{E} \big[ \sum_{t=0}^T \big( f_{t}(x_t) - f_{\delta,t}(x_t) \big) - \sum_{t=0}^T \big(  f_{t}(x) - f_{\delta,t}(x) \big) \big] \nonumber \\
& \leq \frac{1}{2\eta} \|x_0 - x\|^2  + \frac{\eta}{2} \mathbb{E} \big[ \sum_{t=0}^T \|\tilde{g}_t(x_t)\|^2 \big] + 2 \sqrt{d} L_0 \delta T.
\end{align}

Telescoping the bound in \eqref{eqn:SecondMomentBound} over $t = 1, 2,...,T$, adding $\mathbb{E} \big[ \|\tilde{g}_0(x_0) \|^2 \big]$ on both sides, adding $\frac{4 d L_0^2 \eta^2}{\delta^2} \mathbb{E}[ \|\tilde{g}_{T}(x_{T})\|^2]$ to the right hand side and using Assumption~\ref{asmp:BoundVariation}, we obtain that
\begin{align}\label{eqn:BSM_Oline_3}
\mathbb{E} \big[ \sum_{t=0}^T \|\tilde{g}_t(x_t)\|^2 \big] \leq \frac{1}{1-\alpha} \mathbb{E} \big[ \|\tilde{g}_0(x_0)\|^2 \big] + \frac{16}{1 - \alpha} L_0^2 (d+4)^2 T + \frac{2dV_f^2}{1 - \alpha} \frac{1}{\delta^2} T,
\end{align}
where $\alpha = \frac{4 d L_0^2 \eta^2}{\delta^2}$. Substituting the above bound into \eqref{eqn:BSM_Oline_2} yields that
\begin{align}
\mathbb{E} \big[ \sum_{t=0}^T f_{t}(x_t) - \sum_{t=0}^T f_{t}(x) \big]  \leq & \;  \frac{1}{2\eta} \|x_0 - x\|^2 + \frac{\eta}{2(1-\alpha)} \mathbb{E} \big[ \|\tilde{g}_0(x_0)\|^2 \big] + \frac{16}{1 - \alpha} L_0^2 (d+4)^2 \eta T  \nonumber \\
& + 2 \sqrt{d} L_0 \delta T + \frac{2dV_f^2}{1 - \alpha} \frac{\eta}{\delta^2} T.
\end{align}
Since above inequality holds for all $x \in \mathcal{X}$, we can replace $x$ with $x^\ast$. When the upper bound on $\|x_0 - x^\ast\| \leq R$ is known, let $\eta = \frac{R^{\frac{3}{2}}}{2\sqrt{2} L_0 \sqrt{d} T^{\frac{3}{4}}}$ and $\delta = \frac{\sqrt{R}}{T^{\frac{1}{4}}}$, so that $\alpha = \frac{4 d L_0^2 \eta^2}{\delta^2} = \frac{R^2}{2T} \leq \frac{1}{2}$, when $T \geq R^2$. Then, we obtain that
\begin{align} \label{eqn:BSM_Oline_4}
\mathbb{E} \big[ \sum_{t=0}^T f_{t}(x_t) - \sum_{t=0}^T f_{t}(x^\ast) \big] \leq & \; \sqrt{2} L_0 \sqrt{dR} T^{\frac{3}{4}}  + \frac{\mathbb{E} \big[ \|\tilde{g}_0(x_0)\|^2 \big] R^{\frac{3}{2}}}{2 \sqrt{2d} L_0 T^{\frac{3}{4}}} \nonumber \\
& + 8\sqrt{2} \frac{(d+4)^2}{\sqrt{d}} L_0 R^{\frac{3}{2}} T^{\frac{1}{4}} + 2 L_0 \sqrt{dR} T^{\frac{3}{4}} + \frac{\sqrt{2dR} V_f^2}{L_0} T^{\frac{3}{4}}.
\end{align}

When $R$ is unknown, let $\eta = \frac{1}{2\sqrt{2} L_0 \sqrt{d} T^{\frac{3}{4}}}$ and $\delta = \frac{1}{T^{\frac{1}{4}}}$, so that $\alpha = \frac{4 d L_0^2 \eta^2}{\delta^2} = \frac{1}{2T} \leq \frac{1}{2}$. Then, we obtain that
\begin{align} \label{eqn:BSM_Oline_4.1}
\mathbb{E} \big[ \sum_{t=0}^T f_{t}(x_t) - \sum_{t=0}^T f_{t}(x^\ast) \big] \leq & \; \sqrt{2} L_0 \sqrt{d} R^2 T^{\frac{3}{4}}  + \frac{\mathbb{E} \big[ \|\tilde{g}_0(x_0)\|^2 \big]}{2 \sqrt{2d} L_0 T^{\frac{3}{4}}} + 8\sqrt{2} \frac{(d+4)^2}{\sqrt{d}} L_0 T^{\frac{1}{4}} \nonumber \\
& + 2 \sqrt{d} L_0 T^{\frac{3}{4}} + \frac{\sqrt{2d} V_f^2}{L_0} T^{\frac{3}{4}}.
\end{align}

On the other hand, we can let $\eta = \frac{R^{\frac{3}{2}}}{2\sqrt{2} L_0 \sqrt{d} T^{\frac{3}{4}}}$ and $\delta = \frac{\sqrt{R}}{L_0^qT^{\frac{1}{4}}}$, where $q\in\mathbb{R}$ is a user-specific parameter. With this choice of parameters, we get $\alpha = \frac{4 d L_0^2 \eta^2}{\delta^2} = \frac{L_0^{2q}R^2}{2T} \leq \frac{1}{2}$ when $T \geq L_0^{2q}R^2$ and, as a result, we obtain that
	\begin{align} \label{eqn:BSM_Oline_4.2}
	\mathbb{E} \big[ \sum_{t=0}^T f_{t}(x_t) - \sum_{t=0}^T f_{t}(x^\ast) \big] \leq & \; \sqrt{2} L_0 \sqrt{dR} T^{\frac{3}{4}}  + \frac{\mathbb{E} \big[ \|\tilde{g}_0(x_0)\|^2 \big] R^{\frac{3}{2}}}{2 \sqrt{2d} L_0 T^{\frac{3}{4}}} + 8\sqrt{2} \frac{(d+4)^2}{\sqrt{d}} L_0 R^{\frac{3}{2}} T^{\frac{1}{4}} \nonumber \\
	& + 2 L_0^{1-q} \sqrt{dR} T^{\frac{3}{4}} + \sqrt{2dR} L_0^{2q-1}V_f^2 T^{\frac{3}{4}}.
	\end{align}

\section{Proof of Theorem \ref{thm: convex_smooth}}
\label{sec:proof_convexSmooth}
Since $f_t(x) \in C^{1,1}$, we know that $|f_{\delta, t}(x) - f_t(x)| \leq \delta^2 L_1 d$. 
Following the same proof logic as that for proving \eqref{eqn:BSM_Oline_2}, we obtain that 
\begin{align} \label{eqn:BSM_Oline_5}
\mathbb{E} \big[ \sum_{t=0}^T f_{t}(x_t) - \sum_{t=0}^T f_{t}(x) \big] \leq \frac{1}{2\eta} \|x_0 - x\|^2  + \frac{\eta}{2} \mathbb{E} \big[ \sum_{t=0}^T \|\tilde{g}_t(x_t)\|^2 \big] + 2  d L_1 \delta^2 T.
\end{align}
Substituting the bound in \eqref{eqn:BSM_Oline_3} into the above inequality, we obtain that
\begin{align}
\mathbb{E} \big[ \sum_{t=0}^T f_{t}(x_t) - \sum_{t=0}^T f_{t}(x) \big]  \leq & \;  \frac{1}{2\eta} \|x_0 - x\|^2 + \frac{\eta}{2(1-\alpha)} \mathbb{E} \big[ \|\tilde{g}_0(x_0)\|^2 \big] + \frac{16}{1 - \alpha} L_0^2 (d+4)^2 \eta T  \nonumber \\
& +  2  d L_1 \delta^2 T + \frac{2dV_f^2}{1 - \alpha} \frac{\eta}{\delta^2} T. 
\end{align}
Since above inequality holds for all $x \in \mathcal{X}$, we can replace $x$ with $x^\ast$. 
Assuming the bound $\|x_0 - x^\ast\| \leq R$ is known, let $\eta = \frac{R^{\frac{4}{3}}}{2\sqrt{2} L_0 d^\frac{2}{3} T^{\frac{2}{3}}}$ and $\delta = \frac{R^{\frac{1}{3}}}{  d^\frac{1}{6} T^{\frac{1}{6}}}$ so that $\alpha = \frac{4 d L_0^2 \eta^2}{\delta^2} = \frac{R^2}{2T} \leq \frac{1}{2}$ when $T \geq R^2$. Plugging these parameters into above inequality, we finally obtain that
\begin{align}
\mathbb{E} \big[ \sum_{t=0}^T f_{t}(x_t) - \sum_{t=0}^T f_{t}(x) \big]  \leq & \;  \sqrt{2} L_0 d^\frac{2}{3} R^{\frac{2}{3}} T^{\frac{2}{3}} + \frac{\mathbb{E} \big[ \|\tilde{g}_0(x_0)\|^2 \big] R^{\frac{4}{3}}}{2\sqrt{2} L_0 d^\frac{2}{3} T^{\frac{2}{3}}}
+ 8\sqrt{2} L_0 \frac{(d+4)^2}{d^\frac{2}{3}} R^{\frac{4}{3}} T^\frac{1}{3}  \nonumber \\
& +  2 L_1 d^\frac{2}{3}R^{\frac{2}{3}} T^\frac{2}{3} + \frac{\sqrt{2}}{L_0} d^\frac{2}{3} R^{\frac{2}{3}} V_f^2 T^\frac{2}{3}.
\end{align}

When the bound $\|x_0 - x^\ast\| \leq R$ is unknown. Choose $\eta = \frac{1}{2\sqrt{2} L_0 d^\frac{2}{3} T^{\frac{2}{3}}}$ and $\delta = \frac{1}{  d^\frac{1}{6} T^{\frac{1}{6}}}$ so that $\alpha = \frac{4 d L_0^2 \eta^2}{\delta^2} = \frac{1}{2T} \leq \frac{1}{2}$. Plugging these parameters into above inequality, we finally obtain that
\begin{align}
\mathbb{E} \big[ \sum_{t=0}^T f_{t}(x_t) - \sum_{t=0}^T f_{t}(x) \big]  \leq & \;  \sqrt{2} L_0 d^\frac{2}{3} \|x_0 - x\|^2 T^{\frac{2}{3}} + \frac{\mathbb{E} \big[ \|\tilde{g}_0(x_0)\|^2 \big]}{2\sqrt{2} L_0 d^\frac{2}{3} T^{\frac{2}{3}}}
+ 8\sqrt{2} L_0 \frac{(d+4)^2}{d^\frac{2}{3}} T^\frac{1}{3}  \nonumber \\
& +  2  d^\frac{2}{3} L_1 T^\frac{2}{3} + \frac{\sqrt{2}}{L_0} d^\frac{2}{3} V_f^2 T^\frac{2}{3}.
\end{align}
The proof is complete.

\section{Proof of Theorem \ref{thm:Online_Nonconvex_Nonsmooth}}
\label{sec:proof_nonconvexLip}
We first consider the case where Assumption~\ref{asmp:BoundAccumVariation}.1 holds. 
Note that $f_t(x) \in C^{0,0}$. According to Lemma~\ref{lem:GaussianApprox}, $f_{\delta,t}(x)$ has $L_{1,\delta}$-Lipschitz continuous gradient with $L_{1,\delta} = \frac{\sqrt{d}}{\delta} L_0$. Furthermore, according to Lemma 1.2.3 in \cite{nesterov2013introductory}, we have the following inequality
\begin{equation}
\label{eqn:Inequality_1_Online}
\begin{split}
f_{\delta,t}(x_{t+1}) &\leq f_{\delta,t}(x_{t}) +  \langle \nabla f_{\delta,t}(x_t), x_{t+1} - x_t \rangle + \frac{L_{1,\delta}}{2} \|x_{t+1} - x_t\|^2 \\
& = f_{\delta,t}(x_{t}) - \eta \langle \nabla f_{\delta,t}(x_t), \tilde{g}_t(x_t) \rangle + \frac{L_{1,\delta} \eta^2}{2} \|\tilde{g}_t (x_t)\|^2 \\
& = f_{\delta,t}(x_{t}) - \eta \langle \nabla f_{\delta,t}(x_t), \Delta_t \rangle - \eta \|\nabla f_{\delta,t}(x_t)\|^2 + \frac{L_{1,\delta} \eta^2}{2} \|\tilde{g}_t(x_t)\|^2,
\end{split}
\end{equation}
where $\Delta_t = \tilde{g}_t(x_t) - \nabla f_{\delta,t}(x_t)$. According to Lemma~\ref{lem:UnbiasedEstimate_Noiseless}, we know that $\mathbb{E}_{u_t} [ \tilde{g}_t(x_t) ] = \nabla f_{\delta,t}(x_t)$. Therefore, taking expectation over $u_t$ conditional on $x_t$ on both sides of inequality~\eqref{eqn:Inequality_1_Online} and rearranging terms, we obtain that
\begin{align}
\label{eqn:Inequality_2_Online}
& \eta \mathbb{E}[\|\nabla f_{\delta,t}(x_t)\|^2 ] \leq \mathbb{E}[f_{\delta,t}(x_t)] - \mathbb{E}[f_{\delta,t}(x_{t+1})] + \frac{L_{1,\delta} \eta^2}{2} \mathbb{E}[ \|\tilde{g}_t(x_t)\|^2] \nonumber \\
& \leq \mathbb{E}[f_{\delta,t}(x_t)] - \mathbb{E}[f_{\delta,t+1}(x_{t+1})] + \frac{L_{1,\delta} \eta^2}{2} \mathbb{E}[ \|\tilde{g}_t(x_t)\|^2] + \mathbb{E}[f_{\delta,t+1}(x_{t+1})] - \mathbb{E}[f_{\delta,t}(x_{t+1})], 
\end{align}
where the expectation is conditional on $x_t$. Then, we can further condition both sides of \eqref{eqn:Inequality_2_Online} on $x_0$ without changing the sign of inequality, and then apply the tower rule of conditional expectation to make the expectation in \eqref{eqn:Inequality_2_Online} become full expectation.
Telescoping the above inequality over $t = 0,...,T-1$ and dividing both sides by $\eta$, we obtain that
\begin{equation} \label{eqn:BSM_Online_5.1}
\begin{split}
\sum_{t = 0}^{T-1} \mathbb{E}[\|\nabla f_{\delta,t}(x_t)\|^2 ] & \leq \frac{\mathbb{E}[f_{\delta,0}(x_0)] - \mathbb{E}[f_{\delta,T}(x_{T})]}{\eta} + \frac{L_{1,\delta} \eta}{2} \sum_{t = 0}^{T-1} \mathbb{E}[\|\tilde{g}_t(x_t)\|^2] + \frac{W_T}{\eta} \\
& \leq \frac{\mathbb{E}[f_{\delta,0}(x_0)] - f_{\delta,T}^\ast}{\eta} + \frac{L_{1,\delta} \eta}{2} \sum_{t = 0}^{T-1} \mathbb{E}[\|\tilde{g}_t(x_t)\|^2] + \frac{W_T}{\eta},
\end{split}
\end{equation}
where $f_{\delta,T}^\ast$ is the lower bound of the smoothed function $f_{\delta,T}(x)$. $f_{\delta,T}^\ast$ must exist because we assume the orignal function $f_t(x)$ is lower bounded and the smoothed function has a bounded distance from $f_t(x)$ due to Lemma~\ref{lem:GaussianApprox} for all $t$.

Next,  we consider the case where Assumption~\ref{asmp:BoundAccumVariation}.2 holds. Summing the bound in \eqref{eqn:SecondMomentBound} from $t = 1,...,T$, adding $\mathbb{E} \big[ \|\tilde{g}_0(x_0) \|^2 \big]$ on both sides, and adding $\frac{4 d L_0^2 \eta^2}{\delta^2} \mathbb{E}[ \|\tilde{g}_{T}(x_{T})\|^2]$ to the right hand side, we obtain that
\begin{align}\label{eqn:BSM_Oline_6}
\mathbb{E} \big[ \sum_{t=0}^T \|\tilde{g}_t(x_t)\|^2 \big] \leq \frac{1}{1-\alpha} \mathbb{E} \big[ \|\tilde{g}_0(x_0)\|^2 \big] + \frac{16}{1 - \alpha} L_0^2 (d+4)^2 T + \frac{2d}{1 - \alpha} \frac{\widetilde{W}_T}{\delta^2},
\end{align}
Substituting this bound into the inequality~\eqref{eqn:BSM_Online_5.1}, we obtain that
\begin{equation} \label{eqn:Online_Nonconvex_nonsmooth_2}
\begin{split}
\sum_{t = 0}^{T-1} \mathbb{E}[\| \nabla f_{\delta,t}(x_t)\|^2 ] \leq & \; \frac{\mathbb{E}[f_{\delta,0}(x_0)] - f_{\delta,T}^\ast}{\eta} + \frac{W_T}{\eta} + \frac{ \sqrt{d} L_0\eta}{2\delta} \frac{1}{1-\alpha} \mathbb{E} \big[ \|\tilde{g}_0(x_0)\|^2 \big]  \\
&  + \frac{ \sqrt{d} L_0\eta}{2\delta} \frac{16}{1 - \alpha} L_0^2 (d+4)^2 T + \frac{ \sqrt{d} L_0\eta}{2\delta} \frac{2d}{1 - \alpha} \frac{\widetilde{W}_T}{\delta^2}.
\end{split}
\end{equation}
To fullfill the requirement that $|f_t(x) - f_{\delta,t}(x)| \leq \epsilon_f$, we set the exporation parameter $\delta = \frac{\epsilon_f}{d^\frac{1}{2} L_0}$. In addition, let the stepsize be $\eta = \frac{\epsilon_f^{1.5} }{  2\sqrt{2} L_0^2 d^{1.5} T^{\frac{1}{2}}}$. Then, we have that $\alpha =  \frac{4 d L_0^2 \eta^2}{\delta^2} = \frac{\epsilon_f}{2 d T}\leq \frac{1}{2}$ when $T \geq \frac{1}{d \epsilon_f}$. Therefore, we have that $\frac{1}{1 - \alpha} \leq 2$. Substituting this bound and the choices of $\eta$ and $\delta$ into the bound~\eqref{eqn:Online_Nonconvex_nonsmooth_1}, we finally obtain that
\begin{align}
\sum_{t = 0}^{T-1} \mathbb{E}[\| \nabla f_{\delta,t}(x_t)\|^2 ] \leq & \; 2\sqrt{2}L_0^2 \big( \mathbb{E}[f_{\delta,0}(x_0)] - f_{\delta,T}^\ast + W_T \big) \frac{d^{1.5} }{\epsilon_f^{1.5}} T^{\frac{1}{2}} + \frac{\epsilon_f^{\frac{1}{2}} \mathbb{E} \big[ \|\tilde{g}_0(x_0)\|^2 \big]}{2\sqrt{2d T}} \nonumber \\
& + 4\sqrt{2}L_0\epsilon_f^{\frac{1}{2}} \frac{(d+4)^2}{d^{\frac{1}{2}}} T^{\frac{1}{2}} + \frac{L_0^2}{\sqrt{2}} \frac{d^{1.5} \widetilde{W}_T}{\epsilon_f^{1.5} T^{\frac{1}{2}}}.
\end{align}
The proof is complete.

\section{Proof of Theorem \ref{thm:Online_Nonconvex_smooth}}
\label{sec:proof_nonconvexSmooth}
We first consider the case where Assumption~\ref{asmp:BoundAccumVariation}.1 holds. 
Note that when $f_t \in C^{1,1}$ with Lipschitz constant $L_1$, the smoothed function $f_{\delta, t} \in C^{1,1}$ with Lipschitz constant $L_1$. Therefore,
following the proof of Theorem~\ref{thm:Online_Nonconvex_Nonsmooth} but replacing $L_{1,\delta}$ with $L_1$, we obtain that 
\begin{equation} \label{eqn:BSM_Online_Nonconvex_Smooth_1}
\sum_{t = 0}^{T-1} \mathbb{E}[\|\nabla f_{\delta,t}(x_t)\|^2 ] \leq \frac{\mathbb{E}[f_{\delta,0}(x_0)] - f_{\delta,T}^\ast}{\eta} + \frac{L_{1} \eta}{2} \sum_{t = 0}^{T-1} \mathbb{E}[\|\tilde{g}_t(x_t)\|^2] + \frac{W_T}{\eta}.
\end{equation}
Since $f_t \in C^{1,1}$, according to Lemma~\ref{lem:GaussianApprox}, we have that $\|\nabla f_{\delta,t}(x)- \nabla f_t(x)\|\le \delta L_1 (d+3)^{3/2}$. 

Furthermore, we have that
\begin{align} \label{eqn:BSM_Online_Nonconvex_Smooth_2}
& \sum_{t = 0}^{T-1} \mathbb{E}[\| \nabla f(x_t)\|^2 ] = \sum_{t = 0}^{T-1} \mathbb{E}[\| \nabla f(x_t) - \nabla f_{\delta,t}(x_t) + \nabla f_{\delta,t}(x_t)\|^2 ] \nonumber \\
&\leq 2\mathbb{E}[\| \nabla f(x_t) - \nabla f_{\delta,t}(x_t)\|^2] + 2\mathbb{E}[\| \nabla f_{\delta,t}(x_t) \|^2].
\end{align}

Next, we consider the case where Assumption~\ref{asmp:BoundAccumVariation}.2 holds. 
Substituting the bound in \eqref{eqn:BSM_Oline_6} into \eqref{eqn:BSM_Online_Nonconvex_Smooth_1} and using the bound in \eqref{eqn:BSM_Online_Nonconvex_Smooth_2}, we obtain that
\begin{align}
\sum_{t = 0}^{T-1} \mathbb{E}[\| \nabla f(x_t)\|^2 ] \leq & \; 2 \frac{\mathbb{E}[f_{\delta,0}(x_0)] - f_{\delta,T}^\ast}{\eta} + 2 \frac{W_T}{\eta}  +   \frac{L_1}{1-\alpha} \mathbb{E} \big[ \|\tilde{g}_0(x_0)\|^2 \big] \eta + \frac{16L_1}{1 - \alpha} L_0^2 (d+4)^2 \eta T \nonumber \\
&  \; + \frac{2dL_1 \widetilde{W}_T}{1 - \alpha} \frac{\eta}{\delta^2} + 2 L_1^2 (d+3)^3 \delta^2 T,
\end{align}
Choose $\eta = \frac{1}{2 \sqrt{2} L_{0} d^{\frac{4}{3}} T^{\frac{1}{2}}}$ and $\delta = \frac{1}{ d^{\frac{5}{6}} T^\frac{1}{4}}$. Then, $\alpha = \frac{4 d L_0^2 \eta^2}{\delta^2} = \frac{1}{2 \sqrt{T}} \leq \frac{1}{2}$ and $\frac{1}{1 - \alpha} \leq 2$. Substituting these results into the above inequality, we finally obtain that
\begin{align}
\sum_{t = 0}^{T-1} \mathbb{E}[\| \nabla f(x_t)\|^2 ] \leq &\; 4\sqrt{2} L_0 \big( \mathbb{E}[f_{\delta,0}(x_0)] - f_{\delta,T}^\ast + W_T \big) d^{\frac{4}{3}} T^{\frac{1}{2}} +  \frac{L_1 \mathbb{E} \big[ \|\tilde{g}_0(x_0)\|^2 \big] }{\sqrt{2} L_{0} d^{\frac{4}{3}} T^{\frac{1}{2}}} \nonumber \\
& + 8\sqrt{2} L_1 L_0 \frac{(d+4)^2}{d^\frac{4}{3}} T^{\frac{1}{2}}
+ \frac{\sqrt{2} L_1}{L_0}  d^{\frac{4}{3}} \widetilde{W}_T + 2 L_1^2 \frac{(d+3)^3}{d^{\frac{5}{3}}}  T^{\frac{1}{2}}.
\end{align}
The proof is complete.

%% file: Online_Stochastic_Proof.tex
\section{Proof of Lemma~\ref{lem:BoundSecondMoment_Stoch}}
\label{sec:proof_BSMStoch}
Consider the case when $F_t(x, \xi) \in C^{0,0}$ with $L_0(\xi)$. According to \eqref{eqn:GradientEstimate_Noise}, we have that
\begin{equation}
\label{eqn:BSM_stoch_1}
\begin{split}
\mathbb{E}[\|\tilde{g}_t (x_t)\|^2] & = \mathbb{E}[\frac{1}{\delta^2} \big(F_t(x_t + \delta u_t, \xi_t) - F_{t-1}(x_{t-1} + \delta u_{t-1}, \xi_{t-1})\big)^2 \|u_t\|^2]  \\
& \leq \frac{2}{\delta^2} \mathbb{E}[ \big(F_t(x_t + \delta u_t, \xi_t) - F_t(x_{t-1} + \delta u_{t-1}, \xi_t)\big)^2 \|u_t\|^2] \\
& \quad \quad \quad \quad \quad + \frac{2}{\delta^2} \mathbb{E}[ \big( F_t(x_{t-1} + \delta u_{t-1}, \xi_t) - F_{t-1}(x_{t-1} + \delta u_{t-1}, \xi_{t-1})\big)^2 \|u_t\|^2].
\end{split}
\end{equation}
Using the bound in Assumption~\ref{asmp:BoundedVariance} and the fact that the generation of random objective functions $F_{t-1}(\cdot, \xi_{t-1})$ and $F_{t}(\cdot, \xi_{t})$ are independent of $u_t$, we get that $\frac{2}{\delta^2} \mathbb{E}[ \big( F_t(x_{t-1} + \delta u_{t-1}, \xi_t) - F_{t-1}(x_{t-1} + \delta u_{t-1}, \xi_{t-1})\big)^2 \|u_t\|^2] \leq \frac{2d }{\delta^2} V_{f, \xi}^2$. In addition, adding and subtracting $F_t(x_{t-1} + \delta u_{t}, \xi_{t})$ in $\big(F_t(x_t + \delta u_t, \xi_t) - F_t(x_{t-1} + \delta u_{t-1}, \xi_t)\big)^2$ in above inequality, we obtain that
\begin{equation}
\begin{split}
\mathbb{E}[\|\tilde{g}_t (x_t)\|^2] \leq & \frac{4}{\delta^2} \mathbb{E}[ \big(F_t(x_t + \delta u_t, \xi_t) - F_t(x_{t-1} + \delta u_{t}, \xi_t)\big)^2 \|u_t\|^2] \\
& + \frac{4}{\delta^2} \mathbb{E}[ \big( F_t(x_{t-1} + \delta u_{t}, \xi_t) - F_t(x_{t-1} + \delta u_{t-1}, \xi_{t})\big)^2 \|u_t\|^2] \\
& + \frac{2 d}{\delta^2} \mathbb{E}[ \big( F_t(x_{t-1} + \delta u_{t-1}, \xi_t) - F_{t-1}(x_{t-1} + \delta u_{t-1}, \xi_{t-1})\big)^2 ].
\end{split}
\end{equation}
By Lipschitz continuity of $F_t(\cdot;\xi_t)$, we can bound the first two items on the right hand side of above inequality following the same procedure after inequality~\eqref{eqn:BSM_2} and get that 
\begin{align}
\mathbb{E}[\|\tilde{g}_t (x_t)\|^2] \leq & \; \frac{4 d L_0^2 \eta^2}{\delta^2} \mathbb{E}[ \|\tilde{g}_t(x_{t-1})\|^2] + 16L_0^2 (d+4)^2  \nonumber \\
& + \frac{2 d}{\delta^2} \mathbb{E}[ \big( F_t(x_{t-1} + \delta u_{t-1}, \xi_t) - F_{t-1}(x_{t-1} + \delta u_{t-1}, \xi_{t-1})\big)^2 ]. \nonumber
\end{align}
The proof is complete.

%% file: Additional_Proof.tex
\section{Residual-Feedback Convex Optimization with Unit Sphere Sampling} 
\label{sec:UniformSample}
Consider the online bandit optimization problem (P) with convex objective functions and a compact constraint set $\mathcal{X}$. In this section, we assume that the objective function $f(x)$ cannot be queried outside the constraint set $\mathcal{X}$. To satisfy this requirement, we estimate the gradient as
\begin{align}
\widetilde{g}_t(x_t) := \frac{d}{\delta}\big(f_t(x_t + \delta u_t) - f_{t-1}(x_{t-1} + \delta u_{t-1})\big) u_t, \label{eqn:GradientEstimate_US}
\end{align}
where $u_{t-1}$ and $u_t$ are independently and uniformly sampled from the unit sphere  $\mathbb{S} := \{x \in \mathbb{R}^d : \|x\| = 1\}$. Consider the smoothed function $f_\delta(x) = \mathbb{E}_{v \in \mathbb{B}}\big[ f(x + \delta v) \big]$, where the random vector $v$ is uniformly sampled from the unit ball $\mathbb{B} = \{x \in \mathbb{R}^d : \|x\| \leq 1\}$ . Then, we have the following lemma
\begin{lem} \label{lem:I.1}
	The function $\widetilde{g}_t(x_t)$ is an unbiased estimate of the gradient $\nabla f_\delta (x_t)$, i.e.,  $\mathbb{E}\big[ \widetilde{g}_t(x_t) \big] = \nabla f_\delta (x_t)$.
\end{lem}
\begin{proof}
	Since $u_t$ is sampled independently from $x_{t-1}$ and $u_{t-1}$, and $u_t$ has zero mean, it is straightforward to complete the proof by applying Lemma 2.1 in \cite{flaxman2005online}.
\end{proof}
To ensure that the iterates are confined within the constraint set $\mathcal{X}$, we consider the update
\begin{align} \label{eqn:Update_US}
	x_{t+1} = \Pi_{(1 - \xi)\mathcal{X}}\big( x_t - \eta \tilde{g}_t(x_t) \big),
\end{align}
where the set $(1-\xi)\mathcal{X}:= \{(1-\xi)x : \forall x \in \mathcal{X} \}$ is a shrinked version of the original constraint set $\mathcal{X}$. The goal is to select a parameter $\xi$ so that for every $x_\xi \in (1-\xi)\mathcal{X}$, $x_\xi + \delta u \in \mathcal{X}$ for every $u \in \mathbb{S}$. To achieve this, we first make the following assumption that is inspired by \cite{flaxman2005online,bubeck2012regret}.
\begin{asmp}
	There exist contants $r$ and $\bar{r}$ such that $ r \mathbb{B} \subset \mathcal{X} \subset \bar{r} \mathbb{B}$.
\end{asmp}
Then, we have the following lemma.
\begin{lem} \label{lem:I.3}
	If the parameter $\xi$ satisfies $1 \geq \xi \geq \frac{\delta}{r}$, then for every iterate $x_t$ obtained using \eqref{eqn:Update_US}, we have that $x_t + \delta u_t \in \mathcal{X}$ for all $u_t \in \mathbb{S}$.
\end{lem}
\begin{proof}
	When $1 \geq \xi \geq \frac{\delta}{r}$, we get that $\| \delta u \| \leq \xi r$. Therefore, there exists $x' \in r \mathbb{B} \subset \mathcal{X}$ such that the vector $\delta u = \xi x'$.
	Since $x_t \in (1 - \xi) \mathcal{X}$, there exists $x \in \mathcal{X}$ such that $x_t = (1 - \xi) x$, and there exists $x' \in \mathcal{X}$ such that $\delta u = \xi x'$. As a result, we have that $x_t + \delta u = (1-\xi)x + \xi x' \in \mathcal{X}$. This is because set $\mathcal{X}$ is convex.
\end{proof}
Next, we study the regret $R_T := \mathbb{E} \Big[ \sum_{t=0}^{T-1} f_{t}(x_t) - \min_{x \in \mathcal{X}} \sum_{t=0}^{T-1} f_{t}(x) \Big]$ achieved by executing the online update~\eqref{eqn:Update_US} We do so  in the following two steps. First, in Lemma~\ref{lem:I.4}, we provide an upper bound on the difference between the optimal solution that lies in the set $(1-\xi)\mathcal{X}$ and the one that lies in the set $\mathcal{X}$, i.e., $\min_{x \in (1 - \xi)\mathcal{X}} \sum_{t=0}^{T-1} f_{t}(x)  - \min_{x \in \mathcal{X}} \sum_{t=0}^{T-1} f_{t}(x)$; Then, in Theorem~\ref{thm: convex_Lip_US}, we bound the regret defined by the expected difference between the function values achieved by running the update~\eqref{eqn:Update_US} and the term $\min_{x \in (1 - \xi)\mathcal{X}} \sum_{t=0}^{T-1} f_{t}(x)$, i.e., $\mathbb{E} \Big[ \sum_{t=0}^{T-1} f_{t}(x_t) - \min_{x \in (1-\xi)\mathcal{X}} \sum_{t=0}^{T-1} f_{t}(x) \Big]$. Adding the two bounds above, we can complete the proof.

In the following lemma we provide a bound on $\min_{x \in (1 - \xi)\mathcal{X}} \sum_{t=0}^{T-1} f_{t}(x)  - \min_{x \in \mathcal{X}} \sum_{t=0}^{T-1} f_{t}(x)$.
\begin{lem} \label{lem:I.4}
	If the function $f_t$ is convex and $f_t \in C^{0,0}$ with Lipschitz constant $L_0$ for all time $t$, we have that 
	\begin{align} \label{eqn:I.4}
		\sum_{t=0}^{T-1} f_{t}( x_\xi^\ast) - \sum_{t=0}^{T-1} f_{t}(x^\ast) \leq \bar{r} L_0 \xi T,
	\end{align}
	where $x_\xi^\ast = \arg \min_{x \in (1 - \xi)\mathcal{X}} \sum_{t=0}^{T-1} f_{t}(x)$ and $x^\ast  = \arg \min_{x \in \mathcal{X}} \sum_{t=0}^{T-1} f_{t}(x)$.
\end{lem}
\begin{proof}
	Since $x^\ast \in \mathcal{X}$, we have that $(1 - \xi) x^\ast \in (1 - \xi) \mathcal{X}$. Moreover, since $x_\xi^\ast$ is the minimizer in the set $(1- \xi) \mathcal{X}$, we get that
	\begin{align} \label{eqn:I.4_1}
		\sum_{t=0}^{T-1} f_{t}(x_\xi^\ast) \leq \sum_{t=0}^{T-1} f_{t}( (1 - \xi)x^\ast).
	\end{align}
	Also, since $f_t$ is convex and $(1 - \xi)x^\ast = (1 - \xi)x^\ast + \xi 0$, we have that
	\begin{align} \label{eqn:I.4_2}
		f_{t}( (1 - \xi)x^\ast) & \leq (1 - \xi) f(x^\ast) + \xi f_t(0) \nonumber \\
		& \leq  (1 - \xi) f_t(x^\ast) + \xi f_t(x^\ast) - \xi f_t(x^\ast) + \xi f_t(0) \nonumber \\
		& \leq f_t(x^\ast) + \xi L_0 \|x^\ast\| \leq f_t(x^\ast) + \bar{r} L_0 \xi,
	\end{align}
	where the last inequality is due to the fact that $x^\ast \in \mathcal{X} \subset \bar{r} \mathbb{B}$. Summing the inequality~\eqref{eqn:I.4_2} over time, we obtain that
	\begin{align} \label{eqn:I.4_3}
		\sum_{t=0}^{T-1} f_{t}( (1 - \xi)x^\ast) - \sum_{t=0}^{T-1} f_{t}(x^\ast) \leq \bar{r} L_0 \xi T.
	\end{align}
	Adding up the inequalities~\eqref{eqn:I.4_1} and \eqref{eqn:I.4_3} and rearranging terms completes the proof.
\end{proof}
Next, we study the regret $\mathbb{E} \Big[ \sum_{t=0}^{T-1} f_{t}(x_t) - \min_{x \in (1-\xi)\mathcal{X}} \sum_{t=0}^{T-1} f_{t}(x) \Big]$ following similar steps as in Section~\ref{sec:Deterministic_ConvexOnline}. First, we can bound the difference between the smoothed objective function $f_{\delta, t}$ and $f_t$ for every time step $t$ as follows.
\begin{lem}
	Consider a function $f$ and its smoothed version $f_\delta$. It holds that 
	\begin{align*}
	|f_\delta(x)-f(x)|\le
	\begin{cases}
	\delta L_0, ~\text{if}~f\in C^{0,0}, \\
	\delta^2 L_1, ~\text{if}~f\in C^{1,1}.
	\end{cases}
	\end{align*}
\end{lem}
\begin{proof}
	Recall that $f_\delta(x) = \mathbb{E}_{v \in \mathbb{B}} \big[ f(x + \delta v)\big]$. Then, we have that
	\begin{align} \label{eqn:I.5_0}
		|f_\delta(x)-f(x)| & = | \mathbb{E}_{v \in \mathbb{B}} \big[ f(x + \delta v) - f(x)\big] | \nonumber \\
		& \leq \mathbb{E}_{v \in \mathbb{B}} \big[  |f(x + \delta v) - f(x)| \big] \nonumber \\
		& \leq \mathbb{E}_{v \in \mathbb{B}} \big[  L_0 \|\delta v\| \big].
	\end{align}
	Furthermore, since $v \in  \mathbb{B}$, we have that $\|\delta v\| \leq \delta$. Combining this inequality with \eqref{eqn:I.5_0}, we have that $|f_\delta(x)-f(x)| \leq \mathbb{E}_{v \in \mathbb{B}} \big[  \delta L_0 \big] = L_0 \delta$.
	When the function $f \in C^{1,1}$ with Lipschitz constant $L_1$, we have that 
	\begin{align} \label{eqn:I.5_1}
		\langle \nabla f(x), \delta v \rangle - \frac{L_1}{2} \|\delta v\|^2 \leq f(x + \delta v) - f(x) \leq \langle \nabla f(x), \delta v \rangle + \frac{L_1}{2} \|\delta v\|^2,
	\end{align}
	for all $v \in \mathbb{B}$. Taking the expectation of \eqref{eqn:I.5_1} over v sampled uniformly from the unit ball $\mathbb{B}$ and recalling that $v$ is sampled independently from $x$ and has zero mean, we get that
	\begin{align} \label{eqn:I.5_2}
		-L_1\delta^2 \leq -\frac{L_1}{2}\mathbb{E}_{v \in \mathbb{B}}\big[ \|\delta v\|^2 \big] \leq \mathbb{E}_{v \in \mathbb{B}} \big[  f(x + \delta v) - f(x)  \big] \leq \frac{L_1}{2}\mathbb{E}_{v \in \mathbb{B}}\big[ \|\delta v\|^2 \big] \leq L_1 \delta^2.
	\end{align}
	In addition, because $|f_\delta(x)-f(x)| = | \mathbb{E}_{v \in \mathbb{B}} \big[ f(x + \delta v) - f(x)\big] |$, we obtain that $|f_\delta(x)-f(x)| \leq L_1 \delta^2$. The proof is complete.
\end{proof}
The next lemma provides a bound on the second moment of the gradient estimate~\eqref{eqn:GradientEstimate_US} under update~\eqref{eqn:Update_US}. 
\begin{lem}[Second moment]\label{lem:BoundSecondMoment_US}
	Assume that $f_t \in C^{0,0}$ with Lipschitz constant $L_0$ for all time $t$. Then, under the ZO update rule in~\eqref{eqn:Update_US}, the second moment of the residual feedback~\eqref{eqn:GradientEstimate_US} satisfies:
	\begin{align} \label{eqn:SecondMomentBound_US}
	\mathbb{E}[\|\widetilde{g}_t (x_t)\|^2] &\leq \; \frac{4 d^2 L_0^2 \eta^2}{\delta^2} \mathbb{E}[ \|\widetilde{g}_{t-1}(x_{t-1})\|^2] +D_t,
	\end{align}
	where $D_t := 16 d^2 L_0^2 + \frac{2d^2}{\delta^2} \mathbb{E} \big[ \big( f_t(x_{t-1} + \delta u_{t-1}) - f_{t-1}(x_{t-1} + \delta u_{t-1})\big)^2\big]$.
\end{lem}
\begin{proof}
	By definition of the residual feedback~\eqref{eqn:GradientEstimate_US}, we have that
	\begin{equation}
	\label{eqn:BSM_online_1_US}
	\begin{split}
	\mathbb{E}[\|\tilde{g}_t (x_t)\|^2] & = \mathbb{E}[\frac{d^2}{\delta^2} \big(f_t(x_t + \delta u_t) - f_{t-1}(x_{t-1} + \delta u_{t-1})\big)^2 \|u_t\|^2]  \\
	& \leq \frac{2d^2}{\delta^2} \mathbb{E}[ \big(f_t(x_t + \delta u_t) - f_t(x_{t-1} + \delta u_{t-1})\big)^2 \|u_t\|^2] \\
	& \quad \quad \quad \quad \quad + \frac{2d^2}{\delta^2} \mathbb{E}[ \big( f_t(x_{t-1} + \delta u_{t-1}) - f_{t-1}(x_{t-1} + \delta u_{t-1})\big)^2 \|u_t\|^2] \\
	& \leq \frac{2d^2}{\delta^2} \mathbb{E}[ \big(f_t(x_t + \delta u_t) - f_t(x_{t-1} + \delta u_{t-1})\big)^2] \\
	& \quad \quad \quad \quad \quad + \frac{2d^2}{\delta^2} \mathbb{E}[ \big( f_t(x_{t-1} + \delta u_{t-1}) - f_{t-1}(x_{t-1} + \delta u_{t-1})\big)^2],
	\end{split}
	\end{equation}
	where the last inequality is because $u_t \in \mathbb{S}$. Moreover, adding and subtracting $f_t(x_{t-1} + \delta u_{t})$ to the term $\big(f_t(x_t + \delta u_t) - f_t(x_{t-1} + \delta u_{t-1})\big)^2$ in the inequality~\eqref{eqn:BSM_online_1_US}, we obtain
	\begin{equation}
	\begin{split}
	\mathbb{E}[\|\tilde{g} (x_t)\|^2] \leq & \frac{4d^2}{\delta^2} \mathbb{E}[ \big(f_t(x_t + \delta u_t) - f_t(x_{t-1} + \delta u_{t})\big)^2 ] \\
	& + \frac{4d^2}{\delta^2} \mathbb{E}[ \big( f_t(x_{t-1} + \delta u_{t}) - f_t(x_{t-1} + \delta u_{t-1})\big)^2 ] \\
	& + \frac{2d^2}{\delta^2} \mathbb{E}[ \big( f_t(x_{t-1} + \delta u_{t-1}) - f_{t-1}(x_{t-1} + \delta u_{t-1})\big)^2]. 
	\end{split}
	\end{equation}
	Since $f_t \in C^{0,0}$ is Lipschitz with constant $L_0$, we further obtain that
	\begin{align}
	\label{eqn:US_online_2}
	\mathbb{E}[\|\tilde{g} (x_t)\|^2] \leq & \frac{4 d^2 L_0^2}{\delta^2} \mathbb{E}[ \|x_t - x_{t-1}\|^2 ] + 4 d^2 L_0^2 \mathbb{E}[\|u_t - u_{t-1}\|^2] \nonumber \\
	& + \frac{2d^2}{\delta^2} \mathbb{E}[ \big( f_t(x_{t-1} + \delta u_{t-1}) - f_{t-1}(x_{t-1} + \delta u_{t-1})\big)^2].
	\end{align}
	Since $u_t \in \mathbb{S}$, we get that $\mathbb{E}[\|u_t - u_{t-1}\|^2] \leq 4$.  
	Substituting this bound into inequality~\eqref{eqn:US_online_2}, we obtain that
	\begin{align}
	\mathbb{E}[\|\tilde{g} (x_t)\|^2] \leq & \frac{4 d^2 L_0^2}{\delta^2} \mathbb{E}[ \|x_t - x_{t-1}\|^2] + 16 d^2 L_0^2 \nonumber \\
	& + \frac{2d^2}{\delta^2} \mathbb{E}[ \big( f_t(x_{t-1} + \delta u_{t-1}) - f_{t-1}(x_{t-1} + \delta u_{t-1})\big)^2].
	\end{align}
	Since $x_t = \Pi_{(1-\xi)\mathcal{X}} \big[ x_{t-1} - \eta \tilde{g}(x_{t-1}) \big]$, we get that $\|x_t - x_{t-1}\| = \|\Pi_{(1-\xi)\mathcal{X}}\big[ x_{t-1} - \eta \tilde{g}(x_{t-1}) \big] - \Pi_{(1-\xi)\mathcal{X}}\big[ x_{t-1} \big]\| \leq \eta \|\tilde{g}(x_{t-1})\|$ due to the nonexpansiveness of the projection operator onto a convex set. Therefore, we have that
	%
	\begin{align}
	\mathbb{E}[\|\tilde{g}_t (x_t)\|^2] \leq & \frac{4 d^2 L_0^2 \eta^2}{\delta^2} \mathbb{E}[ \|\tilde{g}_{t-1}(x_{t-1})\|^2] + 16 d^2 L_0^2  \nonumber \\
	& + \frac{2d^2}{\delta^2} \mathbb{E}[ \big( f_t(x_{t-1} + \delta u_{t-1}) - f_{t-1}(x_{t-1} + \delta u_{t-1})\big)^2].
	\end{align}
	The proof is complete.
\end{proof}
Using Lemmas~\ref{lem:I.1}-\ref{lem:BoundSecondMoment_US}, we can obtain the main theorem for online convex optimization using~\eqref{eqn:Update_US}.

\begin{thm}[Regret for Convex Lipschitz $f_t$]\label{thm: convex_Lip_US}
	Let Assumption~\ref{asmp:BoundVariation} hold. Assume that $f_t\in C^{0,0}$ is convex with Lipschitz constant $L_0$ for all $t$. Run ZO with residual feedback for $T > \bar{r}^2 L_0^{2q}$ iterations with $\eta = \frac{\bar{r}^{\frac{3}{2}}}{2\sqrt{2} L_0 \sqrt{d} T^{\frac{3}{4}}}$ and $\delta = \frac{\sqrt{\bar{r} d}}{ L_0^q  T^{\frac{1}{4}}}$, where $q \in \mathbb{R}$ is a user-specified parameter. Then, we have that
	\begin{align}
		R_T  \leq & \; 4\sqrt{2\bar{r} d} L_0 T^{\frac{3}{4}} + \frac{\mathbb{E} \big[ \|\tilde{g}_0(x_0)\|^2 \big] \bar{r}^{\frac{3}{2}}}{2 \sqrt{2d} L_0 T^{\frac{3}{4}}} + 8\sqrt{2} d^{\frac{3}{2}} L_0 \bar{r}^{\frac{3}{2}} T^{\frac{1}{4}} \nonumber \\
		&  + (2 + \frac{\bar{r}}{r}) L_0^{1 - q} \sqrt{d\bar{r}} T^{\frac{3}{4}} + \frac{\sqrt{2d\bar{r}} V_f^2}{ L_0^{1 - 2q}} T^\frac{3}{4}.
	\end{align}
	Asymptotically, we have $R_T = \mathcal{O}\big( (L_0 + L_0^{1-q} +{L_0}^{2q-1} V_f^2) \sqrt{d\bar{r}}  T^{\frac{3}{4}} \big)$.
\end{thm}
\begin{proof}
	First, we provide a bound on the regret that compares the sum of the function values obtained using~\eqref{eqn:Update_US} to that obtained for the optimizer $x_\xi^\ast$ in the shrinked constraint set $(1 - \xi)\mathcal{X}$, i.e., $\mathbb{E} \Big[ \sum_{t=0}^{T-1} f_{t}(x_t) - \min_{x \in (1-\xi)\mathcal{X}} \sum_{t=0}^{T-1} f_{t}(x) \Big]$.
	Since $f_{\delta, t}(x)$ is convex for all $t$, we conclude that
	\begin{align} \label{eqn:convexLip_1_US}
	f_{\delta,t}(x_t) - f_{\delta,t}(x) \leq \langle \nabla f_{\delta,t}(x_t), x_t - x \rangle, \text{ for all } x \in (1-\xi)\mathcal{X}.
	\end{align}
	Adding and subtracting $\tilde{g}_t(x_t)$ to $\nabla f_{\delta,t}(x_t)$ in inequality~\eqref{eqn:convexLip_1_US}, and taking the expectation of both sides with respect to $u_t$, we obtain that
	\begin{align}\label{eqn:BSM_Oline_1_US}
	\mathbb{E} \big[ f_{\delta,t}(x_t) - f_{\delta,t}(x) \big] \leq \mathbb{E} \big[ \langle  \tilde{g}_t(x_t), x_t - x \rangle \big].
	\end{align} 
	Since $x_{t+1} = \Pi_{(1-\xi)\mathcal{X}} \big[ x_{t} - \eta \tilde{g}(x_{t}) \big]$, for any $x\in (1-\xi)\mathcal{X}$ we have that
	\begin{align} \label{eqn:convexLip_2_US}
	\|x_{t+1} - x\|^2 & =  \|\Pi_{(1-\xi)\mathcal{X}} \big[ x_{t} - \eta \tilde{g}(x_{t}) \big] - \Pi_{(1-\xi)\mathcal{X}} \big[x \big]\|^2 \nonumber \\
	&\leq \| x_{t} - \eta \tilde{g}(x_{t}) - x \|^2 \nonumber \\
	& = \|x_t - x\|^2 - 2 \eta \langle \tilde{g}_t(x_t), x_t - x \rangle + \eta^2 \|\tilde{g}_t(x_t)\|^2.
	\end{align}
	Rearranging the terms in inequality~\eqref{eqn:convexLip_2_US} yields
	\begin{align} \label{eqn:convexLip_3_US}
	\langle \tilde{g}_t(x_t), x_t - x \rangle \leq \frac{1}{2\eta} \big( \|x_t - x\|^2 -  \|x_{t+1} - x\|^2\big) + \frac{\eta}{2} \|\tilde{g}_t(x_t)\|^2.
	\end{align}
	Taking the expectation of both sides of inequality~\eqref{eqn:convexLip_3_US} with respect to $u_t$ and substituting the resulting bound into \eqref{eqn:BSM_Oline_1_US}, we obtain that 
	\begin{align} \label{eqn:convexLip_4_US}
	\mathbb{E} \Big[ \sum_{t=0}^{T-1} f_{\delta,t}(x_t) - \sum_{t=0}^{T-1} f_{\delta,t}(x) \Big] \leq \frac{1}{2\eta} \|x_0 - x\|^2  + \frac{\eta}{2} \mathbb{E} \Big[ \sum_{t=0}^{T-1} \|\tilde{g}_t(x_t)\|^2 \Big].
	\end{align}
	Since $f_t(x) \in C^{0,0}$, we know that $|f_{\delta, t}(x) - f_t(x)| \leq \delta L_0$. Therefore, we obtain
	\begin{align} \label{eqn:BSM_Oline_2_US}
	\mathbb{E} \big[ \sum_{t=0}^{T-1} f_{t}(x_t) - \sum_{t=0}^{T-1} f_{t}(x) \big] &= \mathbb{E} \big[ \sum_{t=0}^{T-1} f_{\delta,t}(x_t) - \sum_{t=0}^{T-1} f_{\delta,t}(x) \big] \nonumber \\
	& \;\;\;\;\;\; + \mathbb{E} \big[ \sum_{t=0}^{T-1} \big( f_{t}(x_t) - f_{\delta,t}(x_t) \big) - \sum_{t=0}^{T-1} \big(  f_{t}(x) - f_{\delta,t}(x) \big) \big] \nonumber \\
	& \leq \frac{1}{2\eta} \|x_0 - x\|^2  + \frac{\eta}{2} \mathbb{E} \big[ \sum_{t=0}^{T-1} \|\tilde{g}_t(x_t)\|^2 \big] + 2 L_0 \delta T,
	\end{align}
	where we have made use of the bound in \eqref{eqn:convexLip_4_US}. Telescoping the bound in \eqref{eqn:SecondMomentBound_US} over $t = 1, 2,...,T-1$, adding $\mathbb{E} \big[ \|\tilde{g}_0(x_0) \|^2 \big]$ to both sides, and adding $\frac{4 d^2 L_0^2 \eta^2}{\delta^2} \mathbb{E}[ \|\tilde{g}_{T-1}(x_{T-1})\|^2]$ to the right hand side, we obtain that
	\begin{align}\label{eqn:BSM_Oline_3_US}
	\mathbb{E} \big[ \sum_{t=0}^{T-1} \|\tilde{g}_t(x_t)\|^2 \big] \leq \frac{1}{1-\alpha} \mathbb{E} \big[ \|\tilde{g}_0(x_0)\|^2 \big] + \frac{16}{1 - \alpha} d^2 L_0^2 T + \frac{2d^2V_f^2}{1 - \alpha} \frac{1}{\delta^2} T,
	\end{align}
	where $\alpha = \frac{4 d^2 L_0^2 \eta^2}{\delta^2}$. Substituting the bound in \eqref{eqn:BSM_Oline_3_US} into \eqref{eqn:BSM_Oline_2_US} yields
	\begin{align} \label{eqn:convexLip_5_US}
	\mathbb{E} \big[ \sum_{t=0}^{T-1} f_{t}(x_t) - \sum_{t=0}^{T-1} f_{t}(x) \big]  \leq & \;  \frac{1}{2\eta} \|x_0 - x\|^2 + \frac{\eta}{2(1-\alpha)} \mathbb{E} \big[ \|\tilde{g}_0(x_0)\|^2 \big] + \frac{16}{1 - \alpha} d^2 L_0^2 \eta T  \nonumber \\
	& + 2 L_0 \delta T + \frac{2d^2 V_f^2}{1 - \alpha} \frac{\eta}{\delta^2} T.
	\end{align}
	Since inequality~\eqref{eqn:convexLip_5_US} holds for all $x \in (1-\xi)\mathcal{X}$, we can replace $x$ in \eqref{eqn:convexLip_5_US}  with $x_\xi^\ast$. Furthermore, using Lemma~\ref{lem:I.4}, we have that 
	\begin{align} \label{eqn:convexLip_6_US}
				\sum_{t=0}^{T-1} f_{t}( x_\xi^\ast) - \sum_{t=0}^{T-1} f_{t}(x^\ast) \leq \bar{r} L_0 \xi T.
	\end{align}
	Summing inequalities~\eqref{eqn:convexLip_5_US} and \eqref{eqn:convexLip_6_US}, we obtain
	\begin{align} \label{eqn:convexLip_7_US}
		\mathbb{E} \big[ \sum_{t=0}^{T-1} f_{t}(x_t) - \sum_{t=0}^{T-1} f_{t}(x^\ast) \big]  \leq & \;  \frac{1}{2\eta} \|x_0 - x_\xi^\ast\|^2 + \frac{\eta}{2(1-\alpha)} \mathbb{E} \big[ \|\tilde{g}_0(x_0)\|^2 \big] + \frac{16}{1 - \alpha} d^2 L_0^2 \eta T  \nonumber \\
		& + 2 L_0 \delta T + \frac{2d^2 V_f^2}{1 - \alpha} \frac{\eta}{\delta^2} T + \bar{r}L_0\xi T,
	\end{align}	
	where $\|x_0 - x_\xi^\ast\|^2 \leq 4\bar{r}^2$. According to Lemma~\ref{lem:I.3}, we can select $\xi = \frac{\delta}{r}$ to guarantee that all iterates $x_t + \delta u_t \in \mathcal{X}$ for all $u_t \in \mathbb{S}$. Furthermore, let $\eta = \frac{\bar{r}^{\frac{3}{2}}}{2\sqrt{2} L_0 \sqrt{d} T^{\frac{3}{4}}}$ and $\delta = \frac{\sqrt{\bar{r} d}}{ L_0^q  T^{\frac{1}{4}}}$, where $q \in \mathbb{R}$ is a user-specified parameter. Then, $\alpha = \frac{4 d^2 L_0^2 \eta^2}{\delta^2} = \frac{1}{2T} \bar{r}^2 L_0^{2q} \leq \frac{1}{2}$ when $T \geq \bar{r}^2 L_0^{2q}$. Substituting these parameter values into \eqref{eqn:convexLip_7_US}, we obtain that
	\begin{align} \label{eqn:BSM_Oline_4_US}
		\mathbb{E} \big[ \sum_{t=0}^{T-1} f_{t}(x_t) - \sum_{t=0}^{T-1} f_{t}(x^\ast) \big]  \leq & \; 4\sqrt{2\bar{r} d} L_0 T^{\frac{3}{4}} + \frac{\mathbb{E} \big[ \|\tilde{g}_0(x_0)\|^2 \big] \bar{r}^{\frac{3}{2}}}{2 \sqrt{2d} L_0 T^{\frac{3}{4}}} + 8\sqrt{2} d^{\frac{3}{2}} L_0 \bar{r}^{\frac{3}{2}} T^{\frac{1}{4}} \nonumber \\
		&  + (2 + \frac{\bar{r}}{r}) L_0^{1 - q} \sqrt{d\bar{r}} T^{\frac{3}{4}} + L_0^{2q-1}\sqrt{2d\bar{r}} V_f^2 T^\frac{3}{4}.
	\end{align}
	The proof is complete.
\end{proof}

\section{Discussion on Online Optimization with Adversaries}
In Section~\ref{sec:prelim}, we consider online optimization problems where the sequence of the objective functions $\{f_t\}_t$ is randomly generated and is independent of the agent's decisions. This assumption is satisfied when the non-stationarity of the environment is caused by the nature. In this section, we consider a different scenario where the objective function is selected by an opponent. Specifically, at time $t$, the agent selects a decision $x_t + \delta u_t$, then the opponent selects a objective function $f_t$ according to the history information $H_t = \{x_0 + \delta u_0, f_0, \dots, x_{t-1} + \delta u_{t-1}, f_{t-1}, x_t + \delta u_t\}$ to maximize the agent's regret.


When the gradient estimator~\eqref{eqn:GradientEstimate_Noiseless} is applied, where the searching direction $u_t$ is sampled from Gaussian distribution $\mathcal{N}(0, I)$, we have the following Lemma in adversarial scenario.

\begin{lem}[Second moment]\label{lem:BoundSecondMoment_Adv}
	Assume that $f_t \in C^{0,0}$ with Lipschitz constant $L_0$ for all time $t$. Then, under the ZO update rule in~\eqref{eqn:SGD}, the second moment of the residual feedback satisfies: for all $t$,
	\begin{align} \label{eqn:SecondMomentBound_Adv}
	\mathbb{E}[\|\widetilde{g}_t (x_t)\|^2] &\leq \; \frac{4 d L_0^2 \eta^2}{\delta^2} \mathbb{E}[ \|\widetilde{g}_{t-1}(x_{t-1})\|^2] +D_t, \\
	\text{where}~D_t := 16L_0^2 (d+4)^2 + \frac{2}{\delta^2}&\mathbb{E} \big[ \big( f_t(x_{t-1} + \delta u_{t-1}) - f_{t-1}(x_{t-1} + \delta u_{t-1})\big)^2 \|u_t\|^2\big]. \nonumber
	\end{align}
\end{lem}
\begin{proof}
	The proof is essentially the same as the proof of Lemma~\ref{lem:BoundSecondMoment_Det}, except that the bound $\frac{2}{\delta^2}\mathbb{E} \big[ \big( f_t(x_{t-1} + \delta u_{t-1}) - f_{t-1}(x_{t-1} + \delta u_{t-1})\big)^2 \|u_t\|^2\big] \leq \frac{2d}{\delta^2}\mathbb{E} \big[ \big( f_t(x_{t-1} + \delta u_{t-1}) - f_{t-1}(x_{t-1} + \delta u_{t-1})\big)^2\big]$ used under \eqref{eqn:BSM_online_1} does not apply in the adversary case, because the selection of the function $f_t$ depends on $u_t$. Since the other derivations in the proof of Lemma~\ref{lem:BoundSecondMoment_Det} does not rely on the independence between $u_t$ and $f_t$, they still hold. It is straightforward to obtain the bound in \eqref{eqn:SecondMomentBound_Adv}.
\end{proof}

Next, we present the assumptions on the adversary agent for online convex optimization problems.
\begin{asmp}[Bounded Adversary] \label{asmp:BoundVariation_Adv}
	Given the history $H_t$, the adversary agent selects a function $f_t$ such that for all time $t$ there exists a constant $V_f^2$ that satisfies
	\begin{align}
	| f_t(x_{t-1} + \delta u_{t-1}) - f_{t-1}(x_{t-1} + \delta u_{t-1})|^2 \leq V_f^2.
	\end{align}
\end{asmp}

Then, within the expectation term in $D_t$ in the bound~\eqref{eqn:SecondMomentBound_Adv}, for any realization of the random vector $u_t$, the bound $\big( f_t(x_{t-1} + \delta u_{t-1}) - f_{t-1}(x_{t-1} + \delta u_{t-1})\big)^2 \leq V_f^2$ holds according to Assumption~\ref{asmp:BoundVariation_Adv}. Therefore, we have that
\begin{align} \label{eqn:I.1}
\frac{2}{\delta^2}\mathbb{E} \big[ \big( f_t(x_{t-1} + \delta u_{t-1}) - f_{t-1}(x_{t-1} + \delta u_{t-1})\big)^2 \|u_t\|^2\big] \leq \frac{2}{\delta^2}\mathbb{E} \big[ V_f^2 \|u_t\|^2\big] \leq \frac{2d}{\delta^2} V_f^2.
\end{align}
Therefore, after combining Lemma~\ref{lem:BoundSecondMoment_Adv} and Assumption~\ref{asmp:BoundVariation_Adv}, we can achieve the bound on the second moment $\mathbb{E}\big[\|\tilde{g}_t(x_t)\|^2\big]$
\begin{align}
\mathbb{E}\big[\|\tilde{g}_t(x_t)\|^2\big] \leq \frac{4 d L_0^2 \eta^2}{\delta^2} \mathbb{E}[ \|\widetilde{g}_{t-1}(x_{t-1})\|^2] + 16L_0^2 (d+4)^2 + \frac{2d}{\delta^2} V_f^2.
\end{align}
This is the same bound we obtained by combining Lemma~\ref{lem:BoundSecondMoment_Det} and Assumption~\ref{asmp:BoundVariation}. And it can be used to obtain \eqref{eqn:BSM_Oline_3} in the proofs of Theorems~\ref{thm: convex_Lip}, which is also used in \ref{thm: convex_smooth}. Then, it is straightforward to follow the same proofs of Theorems~\ref{thm: convex_Lip} and \ref{thm: convex_smooth} to get the same regret bounds in online convex optimization problems under adversarial environment. 

Finally, we present the assumptions on the adversary agent for non-stationary non-convex optimization problems.

\begin{asmp}\label{asmp:BoundAccumVariation_Adv}
	From time $t = 0$ to $T$, the adversary agent selects a sequence of objective functions $\{f_t\}$ such that
	\begin{enumerate}[leftmargin=*,topsep=0pt,noitemsep]
		\item There exists a constant $W_T$ that satisfies $\sum_{t=1}^{T}\mathbb{E} [ f_{\delta,t}(x_t) - f_{\delta, t-1}(x_t) ] \leq W_T$, where the expectation is taken with respect to $x_t$.
		\item At time $t \geq 1$, given the history $H_t$, the adversary agent selects a function $f_t$ such that there exists a constant $V_{f,t}^2$ that satisfies
		\begin{align}
		| f_t(x_{t-1} + \delta u_{t-1}) - f_{t-1}(x_{t-1} + \delta u_{t-1})|^2 \leq V_{f,t}^2.
		\end{align}
		Furthermore, we have that
		\begin{align}
		\sum_{t=1}^T V_{f,t}^2 \leq \widetilde{W}_T.
		\end{align}
	\end{enumerate}
\end{asmp}
Different from Assumption~\ref{asmp:BoundVariation_Adv}, where at each time $t$, the adversary should select a function $f_t$ according to a uniform function variation bound $V_f^2$, Assumption~\ref{asmp:BoundAccumVariation_Adv}.2 allows the adversary to select $f_t$ according to a varying function variation bound $V_{f,t}^2$. However, there also exists a budget $\widetilde{W}_T$ for the adversary, which represents the total variation on the functions that the adversary is allowed to make from time $t=0$ to $T$.

Then, similar to the discussion under Assumption~\ref{asmp:BoundVariation_Adv}, within the expectation term in $D_t$ in the bound~\eqref{eqn:SecondMomentBound_Adv}, for any realization of the random vector $u_t$, the bound $\big( f_t(x_{t-1} + \delta u_{t-1}) - f_{t-1}(x_{t-1} + \delta u_{t-1})\big)^2 \leq V_{f,t}^2$ holds at time $t$ according to Assumption~\ref{asmp:BoundAccumVariation_Adv}. Therefore, we can combine Lemma~\ref{lem:BoundSecondMoment_Adv} and Assumption~\ref{asmp:BoundAccumVariation_Adv} and use similar derivation in \eqref{eqn:I.1} to achieve the same  bounds in \eqref{eqn:BSM_Online_5.1}, \eqref{eqn:BSM_Oline_6} and \eqref{eqn:BSM_Online_Nonconvex_Smooth_1}, which are used in the proof of Theorems~\ref{thm:Online_Nonconvex_Nonsmooth} and \ref{thm:Online_Nonconvex_smooth}. The other part of the proofs remains the same. Therefore, by combining Lemma~\ref{lem:BoundSecondMoment_Adv} and Assumption~\ref{asmp:BoundAccumVariation_Adv}, we achieve the same regret bounds in Theorems \ref{thm:Online_Nonconvex_Nonsmooth} and \ref{thm:Online_Nonconvex_smooth} in online non-stationary non-convex optimization problems under adversarial environment.

\section{Proof of the Second Moment Bound~\eqref{eqn:secondmoment}}
\label{sec:proof_SecondMoment}
Let $\alpha = \frac{4d L_0^2 \eta^2}{\delta^2}$, using~\eqref{eqn:SecondMomentBound}, we have that
\begin{align}
\mathbb{E}[\|\widetilde{g}_t (x_t)\|^2] \leq \; \alpha^t \mathbb{E}[ \|\widetilde{g}_{0}(x_{0})\|^2] + \sum_{j = 1}^{t} \alpha^{t - j} D_j, \text{ for all } t \geq 1. 
\end{align}
According to Assumption~\ref{asmp:BoundVariation}, we obtain that
\begin{align}
\mathbb{E}[\|\widetilde{g}_t (x_t)\|^2] \leq \; \alpha^t \mathbb{E}[ \|\widetilde{g}_{0}(x_{0})\|^2] + \sum_{j = 1}^{t} \alpha^{t - j} \big( 16L_0^2 (d+4)^2 + \frac{2d}{\delta^2} V_f^2 \big) , \text{ for all } t \geq 1. 
\end{align}
Therefore, we get that
\begin{align}
\mathbb{E}[\|\tilde{g}_t (x_t)\|^2] \leq \max \Big\{ \mathbb{E}[\|\tilde{g}_0 (x_0)\|^2], \dots, \alpha^t \mathbb{E}[ \|\widetilde{g}_{0}(x_{0})\|^2] + \sum_{j = 1}^{t} \alpha^{t - j} \big( 16L_0^2 (d+4)^2 + \frac{2d}{\delta^2} V_f^2 \big), \dots \Big\}. \nonumber
\end{align}
Next, we show that this inequality is equivalent to 
\begin{align}
\mathbb{E}[\|\tilde{g}_t (x_t)\|^2] \leq \max \Big\{ \mathbb{E}[\|\tilde{g}_0 (x_0)\|^2], \frac{1}{1-\alpha} \Big(16L_0^2 (d+4)^2 +  \frac{2d}{\delta^2} V_f^2 \Big) \Big\}.
\end{align}
To see this, observe that the sequence $\Big\{ \mathbb{E}[\|\tilde{g}_0 (x_0)\|^2], \dots, \alpha^t \mathbb{E}[ \|\widetilde{g}_{0}(x_{0})\|^2] + \sum_{j = 1}^{t} \alpha^{t - j} \big( 16L_0^2 (d+4)^2 + \frac{2d}{\delta^2} V_f^2 \big), \dots \Big\}$ is monotonic. This is because if $\mathbb{E}[\|\tilde{g}_0 (x_0)\|^2] \geq \alpha \mathbb{E}[\|\tilde{g}_0 (x_0)\|^2] + 16L_0^2 (d+4)^2 + \frac{2d}{\delta^2} V_f^2$, then we can multiply both sides by $\alpha$ and add $16L_0^2 (d+4)^2 + \frac{2d}{\delta^2} V_f^2$ to both sides and get that $\alpha \mathbb{E}[\|\tilde{g}_0 (x_0)\|^2] + 16L_0^2 (d+4)^2 + \frac{2d}{\delta^2} V_f^2 \geq \alpha^2 \mathbb{E}[\|\tilde{g}_0 (x_0)\|^2] + \alpha\big( 16L_0^2 (d+4)^2 + \frac{2d}{\delta^2} V_f^2 \big) + \big( 16L_0^2 (d+4)^2 + \frac{2d}{\delta^2} V_f^2 \big)$. Using mathematical induction we can show that the sequence is monotonically non-increasing. Similarly, if $\mathbb{E}[\|\tilde{g}_0 (x_0)\|^2] \leq \alpha \mathbb{E}[\|\tilde{g}_0 (x_0)\|^2] + 16L_0^2 (d+4)^2 + \frac{2d}{\delta^2} V_f^2$, then we can show that the sequence is monotonically non-decreasing and converges to $\frac{1}{1-\alpha} \Big(16L_0^2 (d+4)^2 +  \frac{2d}{\delta^2} V_f^2 \Big)$. Therefore, the proof is complete.